\documentclass{article}

\usepackage{graphicx}

\usepackage[utf8]{inputenc} 
\usepackage[T1]{fontenc}    
\usepackage{hyperref}       
\usepackage{url}            
\usepackage{booktabs}       
\usepackage{amsfonts}       
\usepackage{nicefrac}       
\usepackage{microtype}      
\usepackage{xcolor}         

\usepackage{subfigure}
\usepackage{amsmath,amssymb,amsthm}
\usepackage{pifont} 
\usepackage{dsfont}

\usepackage[preprint]{neurips_2026}

\theoremstyle{plain}
\newtheorem{theorem}{Theorem}
\newtheorem{result}[theorem]{Result}
\newtheorem{proposition}[theorem]{Proposition}

\theoremstyle{definition}

\theoremstyle{remark}
\newtheorem{remark}[theorem]{Remark}
\newtheorem{informal}[theorem]{Informal statement}

\newcommand{\E}{\mathbb{E}}
\newcommand{\R}{\mathbb{R}}

\newcommand{\N}{\mathbb{N}}

\newcommand{\1}{\mathds{1}}
\newcommand{\He}{\mathrm{He}}
\newcommand{\var}{\mathrm{Var}}
\newcommand{\supp}{\mathrm{supp}}
\newcommand{\tr}{\mathrm{Tr}}
\newcommand{\va}{\mathsf{v}}

\newcommand{\iid}{\stackrel{\mathrm{i.i.d.}}{\sim}}
\newcommand{\ra}{\rightarrow}

\newcommand{\floor}[1]{\left \lfloor #1 \right \rfloor}

\newcommand{\El}[1]{\mathbb E \left \langle #1 \right \rangle}

\newcommand{\snr}{\mathrm{snr}}

\newcommand{\bx}{\mathbf x}
\newcommand{\bq}{\boldsymbol q}
\newcommand{\bX}{\mathbf X}
\newcommand{\bw}{\mathbf w}
\newcommand{\bW}{\mathbf W}
\newcommand{\bY}{\mathbf Y}
\newcommand{\bv}{\mathbf v}

\newcommand{\bz}{\mathbf z}

\newcommand{\bs}{\mathbf s}
\newcommand{\bI}{\mathbf I}
\newcommand{\by}{\mathbf y}
\newcommand{\bxi}{\boldsymbol{\xi}}

\title{
Sharp feature-learning transitions and Bayes-optimal neural scaling laws in extensive-width networks
}

\author{
  Minh-Toan Nguyen \\
  International Center for Theoretical Physics \\
  Trieste, Italy \\
  \texttt{mnguyen@ictp.it}
  \And
  Jean Barbier \\
 International Center for Theoretical Physics \\
  Trieste, Italy \\
  \texttt{jbarbier@ictp.it}
}

\begin{document}

\maketitle

\begin{abstract}
We study the information-theoretic limits of learning a one-hidden-layer teacher network with hierarchical features from noisy queries, in the context of knowledge transfer to a smaller student model. We work in the high-dimensional regime where the teacher width $k$ scales linearly with the input dimension $d$ -- a setting that captures large-but-finite-width networks and has only recently become analytically tractable.
Using a heuristic leave-one-out argument, validated
numerically throughout, we derive asymptotically sharp characterizations of the Bayes-optimal generalization error and
individual feature overlaps via a system of closed fixed-point
equations. These equations reveal that feature learnability is governed by a sequence of sharp phase transitions: as data grows, teacher features become recoverable sequentially, each through a discontinuous jump in overlap.
This sequential acquisition underlies a precise notion of
\textit{effective width} $k_c$ -- the number of learnable
features at a given data budget $n$ -- which unifies two distinct
scaling regimes: a feature-learning regime in which the Bayes-optimal generalization error
$\varepsilon^{\rm BO}$ scales as $ n^{1/(2\beta)-1}$, and a refinement regime in which it scales as $n^{-1}$, where $\beta>1/2$
is the exponent of the power-law feature hierarchy. Both laws collapse to the single relation $\varepsilon^{\rm BO}=\Theta(k_c d/n)$. We further show empirically that a student trained with \textsc{Adam} near the effective width $k_c$ achieves these optimal scaling laws (up to a small algorithmic gap),
and provide an information-theoretic account of the associated scaling in model size. Finally, we extend the analysis to two hidden layers of extensive widths $k_1,k_2$. There, nothing is learned below $n=\Theta(k_1 d)$, and for $n\gg k_1 d$ we have $\varepsilon^{\rm BO}=\Theta(k_1d/n)$ for every readout. The readout structure determines the effective width of the second layer, but does not affect the scaling of Bayes-optimal error.
\end{abstract}

\section{Introduction}

\begin{figure}[!t]
    \centering
	\subfigure{\includegraphics[width=0.32\linewidth]{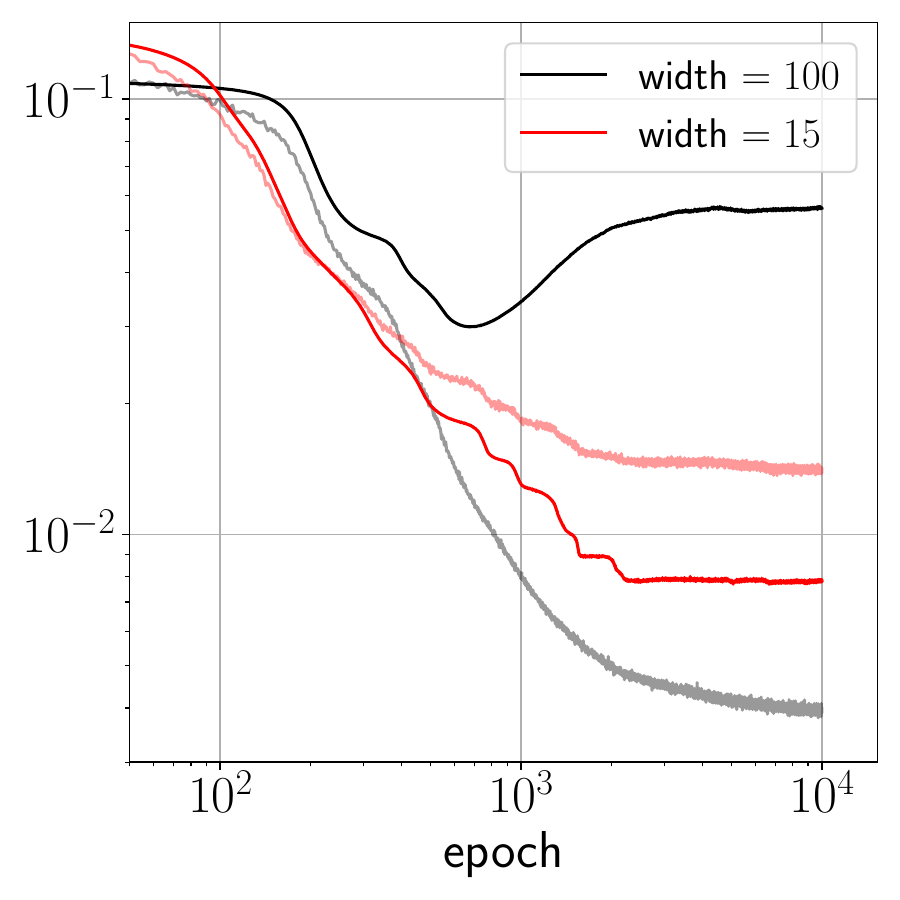}}
    \subfigure{\includegraphics[width=0.32\linewidth]{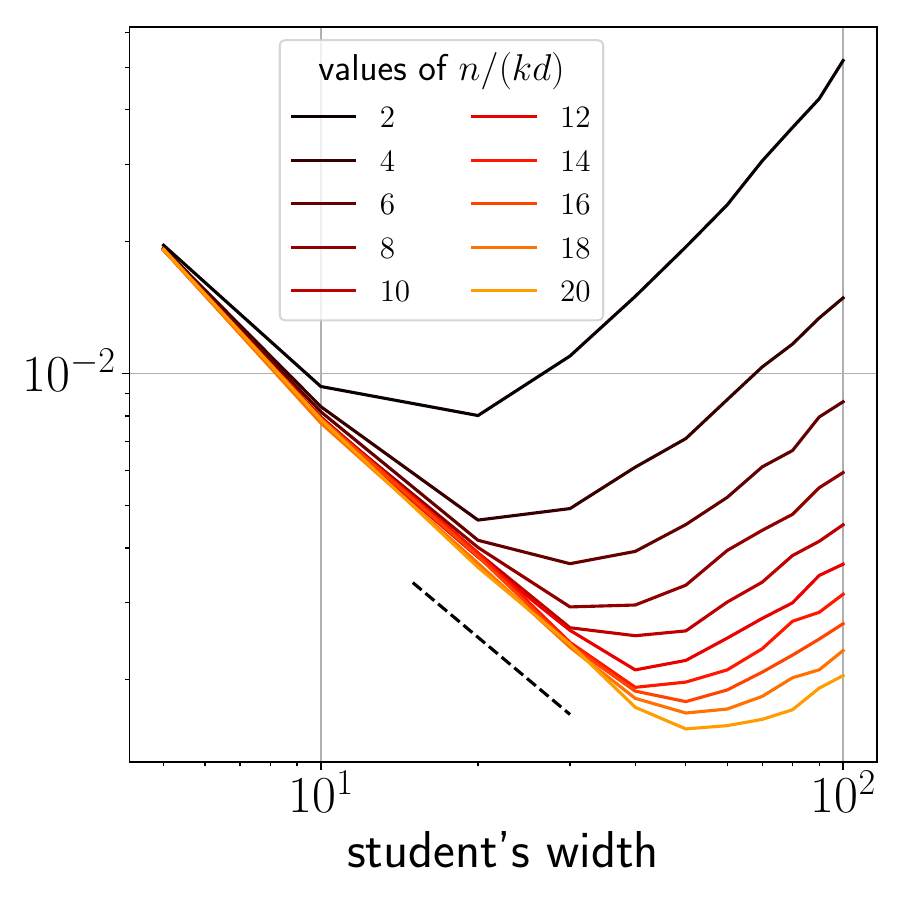}}
    \subfigure{\includegraphics[width=0.32\linewidth]{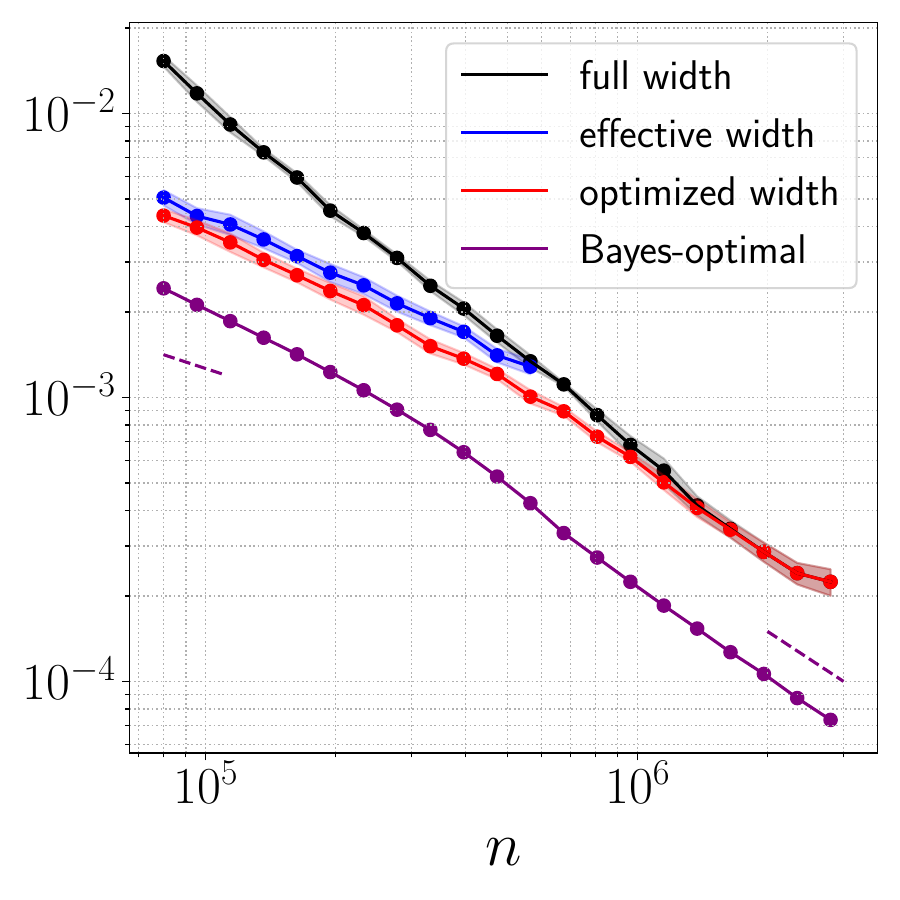}}
     \vspace{-5pt}
   
    \caption{\textsc{Adam}-trained students (all curves except the purple one, right panel) on a teacher with $\sigma(x)=\tanh(2x)$ (with the $2x$ matching the convention in \cite{barbier2025statistical}), $k=100$, $d=200$, $\Delta=0.01$, power-law readout with $\beta=1$ (3 batches, $10^4$ epochs, learning rate $0.003$; averages over 9 runs, shaded regions show $\pm 1$ standard deviation). \textbf{Left:} At fixed sample rate $n/(kd)=2$, a reduced-width student (solid = test error; faint = training error) outperforms the full-width one. \textbf{Middle:} Generalization error vs.\ student width $k_s$ at several sample rates $n/(kd)$. Performance degrades beyond a critical width; the dashed line indicates the theoretically predicted optimal model-size exponent $1-2\beta$. \textbf{Right:} Generalization error vs.\ data size $n$ for three width choices -- \textit{full-width} ($k_s=k$), \textit{effective width} ($k_s=k_c$ from theory), and \textit{optimized width} (best-performing $k_s \in \{10,20,\ldots,100\}$ found by grid search) -- together with
the Bayes-optimal prediction (purple). Dashed lines indicate the predicted scaling exponents: $1/(2 \beta) - 1$ in the \emph{feature-learning regime} and $-1$ in the \emph{refinement regime}, see Section~\ref{sec:discussion} and Appendix~\ref{app:scaling} for the derivation of exponents.}
	\label{fig:sgd-student}
\end{figure}

The performance of large neural networks often follows simple
power laws in data, compute, and model size~\cite{hestness2017deep,
kaplan2020scaling, hoffmann2022training}. A growing line of theory
attributes this regularity to a hierarchical organization of the
target function: features carry geometrically decaying importance
and a learner acquires them roughly in
order~\cite{hutter2021learning, michaud2023quantization,
arous2025learning, ren2025emergence, defilippis2025scaling,
defilippis2026optimal}. In a teacher-student setting -- the
basic primitive behind data-driven model
distillation~\cite{bucilua2006model, hinton2015distilling,
gou2021knowledge} -- this hierarchy raises an immediate design
question: at a fixed data budget, how should one size the student?

Figure~\ref{fig:sgd-student} exhibits two empirical regularities for
an \textsc{Adam}-trained student. \emph{(i)} There is a critical width beyond which the test error \emph{increases} (middle
panel); below this threshold, the error follows a power law in model
size with exponent $1-2\beta$ for a power-law feature hierarchy of
exponent $\beta>1/2$. \emph{(ii)} At or near the optimal width, the error
obeys two distinct power laws in $n$ (right panel): a slow regime with
exponent $1/(2\beta)-1$ and a faster regime with exponent $-1$. These
are striking but mechanism-agnostic: they could reflect properties of
\textsc{Adam}, or fundamental information-theoretic constraints. To
disentangle the two, we lift all computational constraints and adopt
a Bayes-optimal viewpoint, asking:

\textbf{Q1.} \emph{Given a fixed dataset, which features of the
teacher are statistically recoverable?}

\textbf{Q2.} \emph{How does the recoverable subset of features constrain the
scaling law of the Bayes-optimal generalization error
$\varepsilon^{\rm BO}$ in the data size $n$?}

We address \textbf{Q1} and \textbf{Q2} in a one-hidden-layer teacher--student model in
the \emph{extensive-width} regime $k=\Theta(d)$, with $k$ the teacher
width and $d$ the input dimension. This regime captures the
large-but-finite width of practical networks, lies strictly between
the well-studied multi-index ($k=\Theta(1)$) and NTK ($k\gg d$)
extremes, and is the minimal setting for non-trivial feature learning
beyond the linear data scale ($n\gg d$). Our analytical handle is a
heuristic leave-one-out (cavity) argument, validated
numerically throughout.

\paragraph{Paper organization.}
Section~\ref{sec:problem} introduces the model and hypotheses.
The main results -- fixed-point equations for the feature overlaps and
Bayes-optimal error, and the ensuing scaling laws -- are stated in
Section~\ref{sec:results} and analyzed in Section~\ref{sec:discussion},
which covers the phase transitions, the error decomposition, the two
scaling regimes and their unification via $k_c$, and the
\textsc{Adam}-trained student.
The leave-one-out derivation is in Section~\ref{sec:derivations}. Section~\ref{sec:conclusion} concludes. Derivations, numerical methods, and additional asymptotic results are
collected in the Appendices. The code is available at \url{https://github.com/Minh-Toan/scaling-laws}.

\subsection{Problem setup}\label{sec:problem}

\paragraph{Notations.} In addition to the standard $O$-notation, we write $A \simeq B$ if $A - B = o(1)$; $A \sim B$ if $A/B \simeq 1$, and with a slight abuse of notation, we also use $\sim$ to indicate a random variable following a probability law; $A \gtrsim B$ means $A \geq c B$ for some absolute constant $c>0$; $A=\Theta(B)$ if $cB \le A \le CB$ for absolute $c,C>0$; $A=\Theta_p(B)$ if $c,C$ may depend on parameter $p$.

\paragraph{Teacher and dataset.} Given inputs $\bx_\mu$ drawn i.i.d.\ from $\mathcal N(0, \bI_d)$ and noise variables $z_\mu$ that are i.i.d.\ standard Gaussian, the dataset $\mathcal{D}=(\bx_\mu,y_\mu^{\rm out})_{\mu\in[n]}$ consists of input-output pairs with noisy responses $y_\mu^{\rm out} = y_\mu + \sqrt{\Delta}\,z_\mu$ and clean responses generated by the teacher neural network:
\begin{align}\label{eq:model}
    y_\mu = \bv^\top \sigma(\bW^0 \bx_\mu).
\end{align}
$\sigma$ is an element-wise activation, $\bv\in\R^k$ is the
readout vector, and $\bW^0\in\R^{k\times d}$ is the teacher's weight matrix.
Each row $\bw_i^0$ of $\bW^0$ defines a latent direction -- a
\emph{feature} -- used to project the input.

\paragraph{Students.}
We study two types of learners.
The \emph{Bayesian student} knows the data-generation mechanism
(activation $\sigma$, prior $P_0$, noise variance $\Delta$, and readout $\bv$\footnote{Assuming $\bv$ is known to the Bayesian student is a convenient
simplification. However, in the scaling regime we consider, $k=\Theta(d)$, it does not change the Bayes-optimal (BO) error compared to the case where $\bv$ is unknown as proved in \cite{barbier2025statistical}.})
but not the specific realization $\bW^0$.
For a new input $\bx$, it predicts the corresponding teacher's output by computing the posterior mean $\E[\bv^\top \sigma(\bW \bx) \mid \mathcal{D}]$, which achieves the minimum mean square error, also known as the \emph{Bayes-optimal error} $\varepsilon^{\rm BO}$.  On the other hand, the \emph{SGD student} is a one-hidden-layer network of width $k_s\leq k$ and same input dimension $d$, with weights $(\bW,\bv)$ trained with \textsc{Adam}. The $L_2$ norm of each $\bw_i$ is normalized after every epoch to prevent excessive weight growth and improve performance.
We work throughout under the following hypotheses.

\textit{(H1) High-dimensional extensive-width regime.} We work in the scaling regime
\begin{equation}\label{model:regime}
    d \gg 1, \qquad k = \Theta(d), \qquad n \gg d.
\end{equation}
The condition $k=\Theta(d)$ sits strictly between the multi-index
regime ($k=\Theta(1)$), where expressivity is limited, and the NTK regime
($k\gg d$), where the model is lazy~\cite{jacot2018neural,
chizat2019lazy, ghorbani2021linearized}.
It is the minimal setting for non-trivial feature learning at the
super-linear data scale.
The condition $n\gg d$ is necessary to move beyond the linear-model
regime $n=\Theta(d)$~\cite{rev_cui2023, camilli2025information}.

\textit{(H2) Hierarchical readouts.}
We assume $\|\bv\|^2=1$ and
$v_1\geq v_2 \geq \cdots \geq v_k \geq 0$,
so $v_i$
controls the signal-to-noise ratio for feature $i$: features with
larger $v_i$ are easier to recover.
We consider readouts that are \emph{dense} ($\sqrt{k}\,\bv$
has a limiting density); \emph{power-law} ($v_i\propto i^{-\beta}$);
\emph{exponential} ($v_i\propto e^{-ci}$); and \emph{ultra-sparse}
(all but $\Theta(1)$ entries vanish).
To derive scaling laws we focus on
$\beta>1/2$. Power laws are common in real data: natural images have an approximately $1/f^2$ power spectrum~\cite{rev_field1987}, and word frequencies follow Zipf's law~\cite{rev_corral2015}. These examples motivate our hypothesis, although they do not directly measure feature importance. The hypothesis is also a simplification: it neglects interactions between features.

\textit{(H3) Teacher weights.}
$\bW^0$ is drawn uniformly from the Stiefel manifold of $k\times d$
matrices with orthonormal rows ($k\leq d$). The BO error is asymptotically unchanged under i.i.d.\ Gaussian
$\bW^0$; the orthonormal ensemble reduces fluctuations and yields cleaner numerical results.

\textit{(H4) Activation.} $\sigma$ is differentiable, and $ \sigma, \sigma'$ are square-integrable w.r.t. $\mathcal{N}(0,1)$. Moreover, $\sigma$ satisfies $\mu_2 := \E_{z\sim\mathcal{N}(0,1)}
[\sigma(z)(z^2-1)] = 0$, i.e.\ the second Hermite coefficient
vanishes (Appendix~\ref{app:hermite}).
This holds for odd activations and is a condition needed for the ``feature weak coupling'' central to our leave-one-out analysis (as in \cite{ren2025emergence}).
The complementary regime $\mu_2\neq 0$ is studied
in~\cite{maillard2024bayes}.

\subsection{Contributions} \label{sec:contributions}

\textbf{Sharp staircase of feature-learning transitions
(answers Q1; Result~\ref{res:main}.1).}
We derive a closed system of fixed-point equations for the individual
feature overlaps $(q_i)_{i=1}^k$ and the Bayes-optimal error
$\varepsilon^{\rm BO}$, valid for arbitrary readout distributions. Feature $i$ is recoverable ($q_i>0$) if and only if its
self-consistent SNR
$\snr_i = (n/d)\, v_i^2/(\Delta+\varepsilon^{\rm BO})$
exceeds an activation-dependent threshold $\lambda_\sigma$. The $(q_i)$ cross it sequentially as $n$ grows,
producing a staircase of discontinuous overlap jumps and matching
drops in $\varepsilon^{\rm BO}$ (Fig.~\ref{fig:err-overlaps}) -- this mirrors what happens dynamically under SGD training \cite{ren2025emergence}. The jumps are discontinuous because $\mu_2=0$ (Section~\ref{sec:discussion}). The
number $k_c\le k$ of features above threshold defines the
\emph{effective width} of the teacher, i.e. the number of information-theoretically learnable features given a fixed data budget.

\textbf{Bayes-optimal scaling laws unified through $k_c$
(answers Q2; Result~\ref{res:main}.2--3).}
The Bayes-optimal error decomposes transparently as
\(\varepsilon^{\rm BO}
   = \Theta(k_c d/n)
   + \Theta(\sum_{j>k_c} v_j^2)\),
i.e.\ a finite-sample term contributing $\Theta(d/n)$ per learnable
feature plus an irreducible bias from the unlearned ones. For
power-law readouts with $\beta>1/2$ the second term is at most of the same order as the first, collapsing both regimes of
Fig.~\ref{fig:sgd-student}~(right) to the single relation
$\varepsilon^{\rm BO}=\Theta(k_c d/n)$. The two empirical regimes
then arise from the two asymptotics of $k_c(n)$:
$k_c\sim(n/d)^{1/(2\beta)}$ in the \emph{feature-learning} regime
($n\ll k^{2\beta}d$), giving
$\varepsilon^{\rm BO}\sim n^{1/(2\beta)-1}$, and $k_c=k$ in the
\emph{refinement} regime ($n\gg k^{2\beta}d$), recovering the
classical parametric rate $n^{-1}$. Beyond the scaling exponents, the fixed-point equations give exact leading constants and the limiting overlap profile in the near-interpolation regime (Appendix~\ref{app:exact}). The same picture explains the model-size
behavior of the \textsc{Adam}-trained student in
Fig.~\ref{fig:sgd-student}: a width-$k_s$ student near $k_c$
matches the Bayes-optimal data-size exponents, while wider students
overfit non-learnable directions (up to a small residual gap at the
weakest features, which may come from a computational-to-statistical gap or from our training procedure, see Section~\ref{sec:discussion}).

\textbf{A leave-one-out approach beyond dense readouts.}
Methodologically, our cavity argument reduces multi-feature inference
to a sequence of $d$-dimensional generalized linear models, coupled  through the Bayes-optimal error, without
recourse to the replica method and for arbitrary readout structures.
As a special case it recovers the dense-readout result of
\cite{barbier2025statistical} (originally obtained via replicas) and
extends it: we characterize the full decay rate of
$\varepsilon^{\rm BO}$ as a function of readout sparsity, where the
dense case only yields $\Theta(1)$, while it can quickly decrease with $d$ for sparser readouts. The analysis covers all
activations with vanishing second Hermite coefficient ($\mu_2=0$),
in particular all odd activations, complementing the $\mu_2\ne 0$
regime treated in
\cite{defilippis2026optimal,maillard2024bayes}.

\textbf{Two hidden layers.} Starting from the replica-symmetric potential of~\cite{barbier2025statistical}, adapted to structured readouts, we extend the analysis to two hidden layers of widths $k_1,k_2=\Theta(d)$ (Result~\ref{res:2layer}). Nothing is learned below $n=\Theta(k_1d)$, whatever the readout; at this scale, the first layer and the top second-layer features are learned together, in a single jump. Beyond it, $\varepsilon^{\rm BO}=\Theta(k_1d/n)$ for every readout: the first layer sets the exponent $-1$, and the readout only changes the constant. The effective width of the second layer follows the one-layer law, with $d$ replaced by $k_1$. HMC simulations agree with the predictions (Fig.~\ref{fig:2layer-hmc}). The calculation for this case is similar to that of the one-hidden-layer model, with details in Appendix~\ref{sec:2layer}.

\begin{figure}
    \centering
    \subfigure{\includegraphics[width=0.47\linewidth]{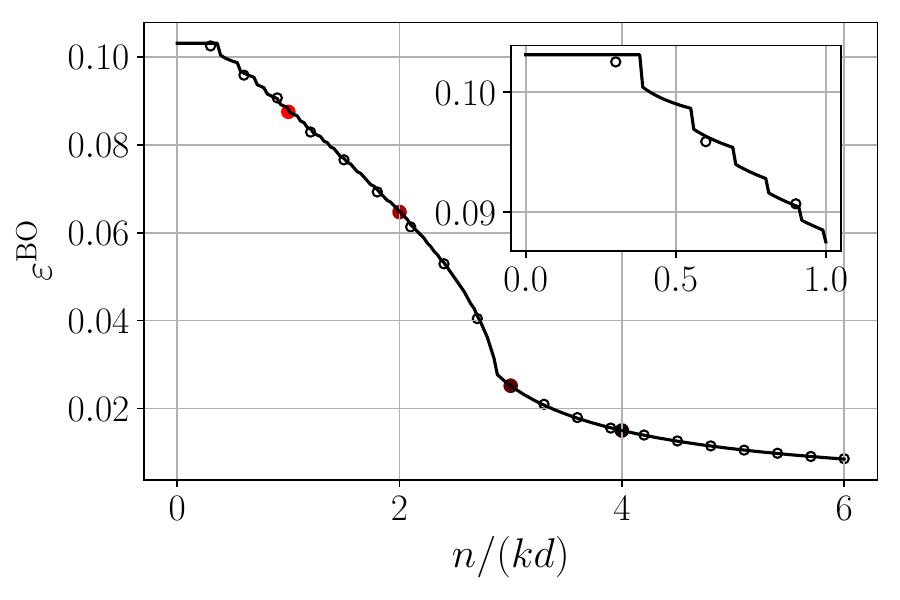}}
    \subfigure{\includegraphics[width=0.47\linewidth]{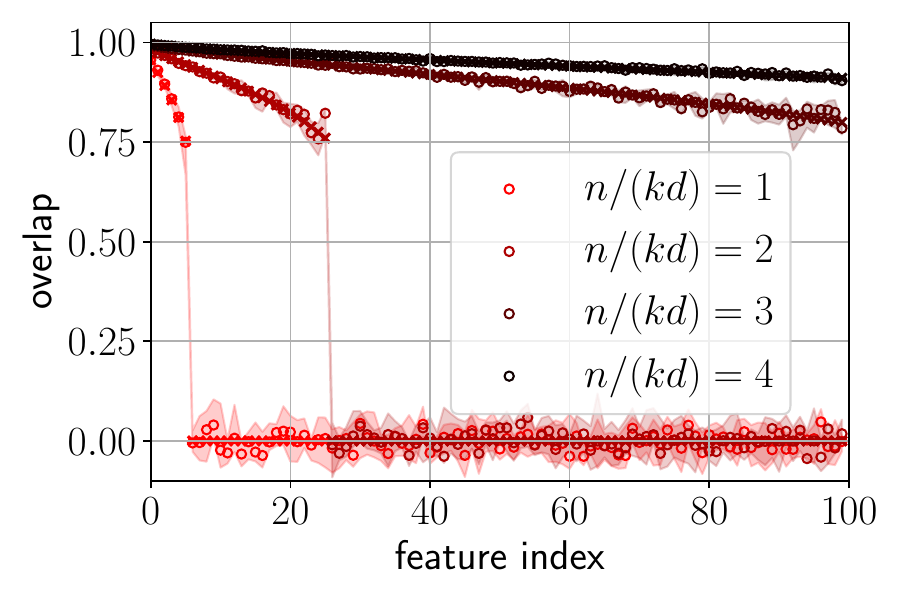}}
    \vspace{-10pt}
    \caption{Bayes-optimal learning curve (\textbf{left}) and feature
overlaps at selected values of $n/(kd)$ (\textbf{right}) for $\sigma(x)=\tanh(2x)$, $d=200$, $k=100$, $\Delta=0.04$, power-law readout with $\beta=0.3$. Solid lines and crosses: theoretical predictions from the
fixed-point equations~\eqref{eq:fpe}; circles: averages over 9 independent Hamiltonian Monte Carlo experiments; shaded regions: $\pm 1$ standard
deviation. Markers on the learning curve indicate the $n/(kd)$ values at
which the overlap profiles on the right are evaluated. The inset shows the staircase of discontinuous drops in $\varepsilon^{\rm BO}$ as successive features are learned (see \S\ref{sec:contributions}). For $\beta<1/2$, the error decay just before entering the refinement regime
is faster than $n^{-1}$; see Section~\ref{sec:discussion} for the comparison with the $\beta>1/2$ case.}
	\label{fig:err-overlaps}
\end{figure}

\subsection{Additional related works}
\textbf{Multi-index models.} 
Multi-index models (teacher width $k = \Theta(1)$ in the context of neural networks) are a well-developed
testbed for feature learning in SGD~\cite{arous2021online,
abbe2023sgd, damian2022neural}, Bayes-optimal~\cite{barbier2019optimal,
aubin2018committee, troiani2025fundamental}, and
computationally-hard~\cite{damian2023smoothing, damian2024computational}
regimes. Despite these important insights, multi-index models remain weakly expressive, and a recent substantial technical effort has been devoted to going beyond this fixed-width setting \cite{cui2021generalization, camilli2025information, maillard2024bayes, ren2025emergence, defilippis2025scaling, barbier2025statistical}. The present work contributes to this research line.

\textbf{Neural scaling laws.} 
Neural scaling laws have been studied from two broad perspectives. A first line of work operates in the non-feature-learning regime -- kernel methods, random feature models, or linear regression -- and derives scaling laws from spectral assumptions on the data covariance. Several of these works obtain scaling laws in both data size and model size \cite{bahri2024explaining, bordelon2024dynamical, lin2024scaling}, or characterize how the optimal model size scales with compute budget \cite{paquette20244+}. A second line of work introduces genuine feature learning, either through abstract task models \cite{hutter2021learning, michaud2023quantization, fonseca2024exactly, cagnetta2025learning} or through concrete neural network models in the multi-index or sublinear-width regime \cite{arous2025learning, ren2025emergence, defilippis2025scaling, defilippis2026optimal}. The quantization model of~\cite{michaud2023quantization} is
phenomenologically close to our framework: it postulates that
networks learn discrete units of knowledge sequentially, in order of
their contribution to loss, and that model capacity caps the number of
learnable units. These properties instead emerge from first principles in
our model.

Our closest predecessor~\cite{defilippis2026optimal} derives data-size scaling laws in the multi-index regime $k=\Theta(1)$ (numerically: $k=10$, $d=1000$, vs.\ $k=100$, $d=200$ here), using a hand-crafted rather than standard generalization error, and restricted to activations with $\mu_2\neq 0$; ours are complementary, covering $\mu_2=0$ and
hence all odd activations. On the algorithmic side, the optimal scaling laws in~\cite{defilippis2026optimal} are achieved by AMP, later extended in~\cite{defilippis2026noise} to activations with $\mu_2\neq 0$ and even to even functions with $\mu_2=0$: for AMP, the data-limited exponent acquires a multiplicative correction from a noise-sensitive exponent defined therein, while the information-theoretic exponents are
conjectured to be unchanged. Despite all these differences, these three works obtain the same BO data-size scaling exponent, suggesting universality across activation classes and width regimes. The scaling exponents, the error decomposition and the interpretation of $k_c$ as the number of learnable features already appear in, or are anticipated by,~\cite{defilippis2026optimal,defilippis2025scaling}; similar decompositions also appear in minimax rates for weakly sparse linear regression~\cite{rev_raskutti2011}. Our contributions are a sharp characterization of the individual feature overlaps in the extensive-width regime, for general readouts and for the complementary class of activations with $\mu_2=0$, and the interpretation of $k_c$ as an effective model size, which guides the choice of the student width (Fig.~\ref{fig:sgd-student}).

\textbf{Bayes-optimal learning of extensive-width networks.}
In the linear data regime $n =\Theta(d)$, the BO problem reduces to linear
regression~\cite{cui2021generalization, camilli2025information}; a related effective-network interpretation was obtained in this regime by~\cite{rev_cui2023}.
For quadratic activations at $n =\Theta(d^2)$, \cite{maillard2024bayes}
characterized the BO error but found no feature recovery, due to 
rotational symmetry intrinsic to this specific setting.
The general case was resolved by~\cite{barbier2025statistical} using
the replica method from physics, for dense readouts and multiple layers.
Our work specializes to one hidden layer and $\mu_2=0$, but generalizes
to arbitrary readout structure. Moreover, the present leave-one-out approach, though remaining heuristic, is arguably more direct and does not require familiarity with the replica formalism.

\section{Main results} \label{sec:results}
Before stating the results, we introduce the following definitions. For an input $\bx$, we denote the corresponding response by $y(\bx, \bW^0) := \bv^{\top} \sigma(\bW^0 \bx)$, which emphasizes the dependence of the label on $\bW^0$ (the readout $\bv$ being fixed). Given a response predictor $\hat y(\bx, \mathcal D)$, we use the standard definition of mean square (prediction) error (MSE): $\E_{\mathcal D, \bW^0}\,
\E_{\bx \sim \mathcal N(0, \bI_d)}
(y(\bx, \bW^0) - \hat y(\bx, \mathcal D))^2$. Recall that the Bayes-optimal error $\varepsilon^{\mathrm{BO}}$ is attained by the Bayesian student $\hat y(\bx, \mathcal D) = \E [y(\bx, \bW)\mid\mathcal D]=:\langle y(\bx, \bW)\rangle$, where the bracket notation $\langle \,\cdot\,\rangle$ is used to denote an average w.r.t. the Bayesian posterior distribution: $\bW\sim P( \bW \mid\mathcal D)\propto \nu(\bW) \prod_{\mu\le n}\exp(-\frac{1}{2\Delta} (y_\mu^{\rm out} - y(\bx_\mu,\bW))^2)$, with $\nu(\,\cdot\,)$ the uniform measure over the Stiefel manifold of $k \times d$ matrices with orthonormal rows.

For any function $\sigma \in L_g^2:=L^2(\mathbb R,\mathcal N(0,1))$, define the kernel $g_\sigma$ on $[0,1]$ and $m_\sigma$ on $[0, \infty)$ as
\begin{align}
g_\sigma(x) &:= \E\big[\sigma(z_1)\sigma(z_2)\big], \quad (z_1,z_2) \sim {\mathcal N\!\left(0, \begin{bmatrix} 1 & x \\ x & 1 \end{bmatrix} \right)}, \label{defs0} \\
m_\sigma(\lambda) &:= \arg \max_{q \in [0,1)} \big\{\lambda g_\sigma(q) + q + \ln(1-q)\big\} .\label{defs}
\end{align}
We also define $ \lambda_\sigma := \sup \{ \lambda \geq 0: m_\sigma(\lambda) = 0 \}$. Basic properties of the functions $g_\sigma$ and $m_\sigma$ are established in Appendix \ref{app:msig}. We will see in Section~\ref{sec:derivations} that $m_\sigma(\lambda)$ is related to the overlap between signal and Bayes estimator of a single feature/latent direction in the context of a generalized linear model in a ``high sample rate'' regime \cite{barbier2019optimal}, see Result~\ref{res:glm-highnoise}. It is non-decreasing with a
discontinuity at $\lambda_\sigma$ 
(Appendix~\ref{app:msig}); it is well defined away from transition points, where the maximizer in~\eqref{defs} is unique. $\lambda_\sigma$ is thus the critical SNR below which no estimator can recover the feature.

\paragraph{A note on rigor.}
The extensive-width regime $k=\Theta(d)$ of Bayesian neural networks is particularly challenging to analyze, with only a few rigorous results in simpler settings, see e.g. \cite{camilli2025information}. Our derivations are therefore non-rigorous, but based on powerful heuristic approaches: the cavity (leave-one-out) method and the
hypothesized Gaussianity of so-called cavity fields -- techniques standard in statistical physics that are proved rigorously in simpler related
models~\cite{talagrand2010mean, barbier2019optimal, Montanari_SE} and
conjectured to extend here.
They rely on concentration properties of the posterior; in particular,
\begin{align}
    \El{(\bw^0_i\cdot\bw^1_j)^2} = O(d^{-1}) \ \text{ for } i\neq j,
    \qquad
    \El{(\bw^0_i\cdot\bw^1_i - q_i)^2} = O(d^{-1}),
    \label{concentration-2}
\end{align}
for $\bW^1\sim P(\,\cdot\mid\mathcal{D})$. In words, we assume that the overlaps between distinct teacher and student features are small (\emph{feature weak coupling}), and that the overlap of each feature concentrates to a deterministic limit $q_i$, possibly zero. We prove the first identity in~\eqref{concentration-2} for dense readouts $\sqrt{k}\bv=\1$ in Appendix~\ref{app:concentration}; the proof relies on the exchangeability of the features, and extending it to structured readouts requires new ideas. The first identity is motivated by multi-index models, where the overlaps between distinct features are of order $k^{-1}$: for $k=\Theta(d)$, we expect them to have mean $O(k^{-1})$ and variance $O(d^{-1})$, which gives the first identity. The second identity is motivated by concentration results for overlaps in related inference problems~\cite{barbier2022strong,rev_barbier2019}. We assume that these concentrations continue to hold in the more general scaling regimes considered in the paper without proof. The predictions agree with Hamiltonian Monte Carlo (HMC) simulations in all tested parameter regimes, already at
$d\sim 10^2$. Note, however, that HMC is initialized with the first $k_c$ rows equal to those of the teacher and the other rows random (Appendix~\ref{app:hmc}). It thus checks the predicted overlaps and errors, but not the value of $k_c$ itself, nor that the posterior concentrates around this solution.

Before stating the results, let us mention that if one removes the linear component of the activation $\sigma$, i.e. replace it by $\tilde \sigma(x) := \sigma(x) - \mu_0 - \mu_1 x$, then the overlaps $q_{i}$ and Bayes-optimal error we derive remain asymptotically the same (Appendix \ref{app:linear-part}). Consequently, in the subsequent claims, we assume without loss of generality that $\mu_0 = \mu_1= 0$ in addition to our assumption $\mu_2 =0$; in other words, $\sigma$ has \emph{information exponent} at least $3$ \cite{arous2021online}, as in the complementary dynamical analysis of \cite{ren2025emergence}.

\begin{result} \label{res:main}
Consider the setting described in Section~\ref{sec:problem}. For $i \in \N$, let $\mu_i$ be the $i$-th Hermite coefficient of $\sigma$ (Appendix \ref{app:hermite}). 
\begin{enumerate}
\item \textbf{Feature overlaps and Bayes-optimal error:} Overlaps $(q_i)_{i=1}^k$ solve the fixed-point system
\begin{align}
q_{i} &= m_\sigma \Big( \frac{n v_i^2/d}{\Delta + \varepsilon} \Big), \qquad \varepsilon = \sum_{i=1}^k v_i^2 (g_\sigma(1) - g_\sigma(q_{i})) \label{eq:fpe}
\end{align}
and the BO error satisfies $\varepsilon^{\rm BO} \sim \varepsilon$. Feature $i\in[k]$ is \emph{learnable} if $q_i>0$. Denote by $k_c=k_c(k,d,n)=|\{i : q_i>0\}|$ the \emph{effective width} of the teacher.

\item  \textbf{Asymptotic decomposition of the BO error.} The optimal error decomposes as
\begin{align}\label{eq:err-asymp}
\varepsilon^{\rm BO}
= \Theta\Big(\frac{k_c d}{n}\Big)
+\Theta\Big(\sum_{j>k_c} v_j^2\Big).
\end{align}
The first term corresponds to the estimation error on learnable features; the second to the irreducible error from non-learnable features.

\item \textbf{Bayes-optimal scaling law.}  Fix $k,d$ for the teacher (and for the Bayes-optimal student). Consider power-law distributed readouts with
exponent $\beta>1/2$. From eq.~\eqref{eq:err-asymp} and the characterization of $k_c$ in
Appendix~\ref{app:scaling}, the Bayes-optimal error and effective width obey the following scalings:
\begin{align*}
\varepsilon^{\rm BO} =
\begin{cases}
\Theta_{k,d}(n^{\frac{1}{2\beta}-1}) \\
\Theta_{k,d}(n^{-1}) 
\end{cases}  
\! \! \!\! \! k_c =
\begin{cases}
\Theta((n/d)^{1/(2\beta)}) & n \ll k^{2\beta}d
\quad \text{(feature-learning regime)},\\
k & n \gg k^{2\beta}d
\quad \text{(refinement regime)}.
\end{cases}
\end{align*}

\end{enumerate}
\end{result}

The numerical method used to solve \eqref{eq:fpe} is described in Appendix~\ref{app:fpe}. This procedure consistently finds a single solution. Uniqueness can be easily checked numerically via the one-parameter equation
\eqref{eq:fixedPointUnique} over $\varepsilon$. Uniqueness for $\varepsilon$ then implies the one for the $(q_i)$ via \eqref{eq:fpe}. We checked uniqueness in all our experiments. In general, however,~\eqref{eq:fpe} can have several solutions; Appendix~\ref{app:potential} gives an example and a way to select the physically relevant one.

\begin{remark}[Statistical compressibility]\label{rem:compress}
When $k_c < k$ the teacher is \emph{compressible}.  The same leave-one-out Gaussian argument
(Section~\ref{sec:derivations}) suggests that the $n$ teacher
responses are statistically indistinguishable -- in the sense of
equal Bayes-optimal error -- from those of the reduced model
\begin{align}\label{eq:compressed}
  y_\mu = \bv_{1:k_c}^\top\sigma(\bW^0_{1:k_c}\bx_\mu)
  + \sqrt{\Delta+\Delta'}\,z_\mu,
  \qquad
  \Delta' := g_\sigma(1)\sum_{j>k_c} v_j^2.
\end{align}
Indeed, by definition of $k_c$, the terms $\sum_{i > k_c} v_i \sigma(\bw_i^0 \cdot \bx_\mu)$ are not learnable and act as
effective noise with variance $\Delta'$. When $k-k_c \gg 1$, a central limit/Gaussianity assumption applied to this noise and
combined with the label noise yields the reduced model~\eqref{eq:compressed}. This is the information-theoretic basis for the empirical finding that a reduced-width student can outperform a full-width
one (Fig.~\ref{fig:sgd-student}, left). On the other hand, when $k-k_c = O(1)$, the Gaussian approximation does not hold for the noise induced by the non-learnable terms. However, this contribution is negligible compared to $\Delta$.
\end{remark}

\section{Discussion}\label{sec:discussion}

\textbf{Sharp staircase learning transitions.} Consider the system \eqref{eq:fpe}. The transitions are discontinuous because the activation has information exponent at least three ($\mu_2=0$ and the linear part is removed). This fails for general activations: in the single-feature GLM, ReLU has no transition and $|x|$ a continuous one (Fig.~\ref{fig:glm}). Since $m_\sigma(x)=0$ if $x \leq \lambda_\sigma$ by definition, feature $i$ is learnable if and only if 
\begin{align}\label{eq:learnable-cond}
    \snr_i: =  \frac{(n/d) v_i^2}{\Delta + \varepsilon} > \lambda_\sigma.
\end{align}
The \emph{effective signal-to-noise ratio} $\snr_i$ entering \eqref{eq:fpe} is determined by the $i$-th feature strength $v_i^2$ (which depends on $k$, e.g. it scales as $1/k$ in the dense case), the sampling rate $n/d$, and the effective noise variance $\Delta + \varepsilon$. As $n$ grows, and assuming that $v_i>v_{i+1}$ strictly, $(\snr_i)_i$ grow and cross the threshold \eqref{eq:learnable-cond} sequentially -- thus increasing $k_c$ --, leading to a staircase sequence of phase transitions, see Fig.~\ref{fig:err-overlaps} (left panel) and \cite{ren2025emergence,defilippis2026optimal} for similar phenomena in settings of sub-linear width ($k\ll d$). 

\textbf{Decomposition of the Bayes-optimal error.} Equation~\eqref{eq:err-asymp} has a transparent interpretation. Firstly, the underlying mechanism is that each learned feature contributes $\Theta(d/n)$ to the generalization error, while a non-learned feature $i$ contributes $\Theta(v_i^2)$. Secondly, writing $\varepsilon^{\rm BO}=\Theta(k_c d/n)$ when the second term is at most of the same order as the first, the BO error scales as the inverse of the
\emph{effective sampling rate} $n/(k_c d)$: the teacher behaves as
a model with $k_c$ features, each requiring $\Theta(d)$ samples to
estimate.
The data-size scaling laws of $\varepsilon^{\rm BO}$ in both regimes thus
follow directly from how $k_c$ depends on $n$ (Fig.~\ref{fig:sgd-student}).

\textbf{Bayes-optimal scaling law.} With power-law readouts with $\beta>1/2$, the crossover at $n^\star = \Theta(k^{2\beta}d)$ in the scaling law for $\varepsilon^{\rm BO}$ marks the point at
which the weakest feature $k$ becomes learnable,
i.e. the smallest $n$ s.t. \
$(n/d)v_k^2 =\Theta((n/d)k^{-2\beta})  \gtrsim \lambda_\sigma$.
For $n\ll n^\star$, only $k_c = \Theta((n/d)^{1/(2\beta)})\ll k$
features are above threshold and are learned sequentially; with $k_c$ defined by the condition that $\snr_{k_c}\propto (n/d)k_c^{-2\beta}$ is greater than the threshold $\lambda_\sigma$ while $\snr_{k_c+1}$ is not, see Appendix~\ref{app:scaling} for more details. Instead, when $n\gg n^\star$ all features are learned (i.e. have overlap $q_i>0$) and additional data only refines them, yielding the feature-learning vs. refinement regimes, see Fig.~\ref{fig:sgd-student}.

Still with $\beta>1/2$, the second error term in \eqref{eq:err-asymp} corresponding to non-learnable features is of order smaller or equal to the error corresponding to learnable features, see Appendix~\ref{app:scaling}. Consequently, both scaling laws unify into the single expression involving the effective sampling rate $n/(k_c d)$:
\begin{equation}
    \varepsilon^{\rm BO} = \Theta(k_c d / n).
\end{equation} This relation holds up to constants, which differ between the two regimes; it unifies the two scaling laws at the level of exponents. Exact constants are derived in the near-interpolation regime in Appendix~\ref{app:exact}. In the feature-learning regime $k_c$ grows with $n$, which partially
offsets the $1/n$ decay and yields the slower exponent
$1/(2\beta)-1 > -1$; in the refinement regime $k_c = k$ is fixed,
recovering the $n^{-1}$ rate.
The effective width $k_c$ thus provides a unified account of both
scaling regimes. These scaling laws provide a concrete theoretical instantiation of
the quantization model of~\cite{michaud2023quantization}: their
\textit{capacity hypothesis} -- that the number of learnable quanta grows
with model size -- emerges here directly from the width bottleneck
below $k_c$ and the sharp SNR threshold $\lambda_\sigma$.

We emphasize that the analytic formula $k_c = \Theta((n/d)^{1/(2\beta)})$ in the feature-learning regime
provides a first-principles prescription for the optimal student width directly from the data budget, without retraining at each candidate model size -- a concrete theoretical underpinning for \textit{hyperparameter transfer} of model width~\cite{yang2021tuning}: rather than fitting the optimal size empirically by training multiple students, the exponent $\beta$ of the feature hierarchy determines how to optimally scale the model width $k_s$ as $n$ varies, in the same spirit as compute-optimal scaling laws~\cite{hoffmann2022training}. 

As a final remark, although the scaling within the refinement regime is $n^{-1}$ in both cases, the behavior of the BO error just before entering this regime differs between $\beta > 1/2$ and $\beta < 1/2$. When $\beta < 1/2$ and $i \sim k$, we have $v_i^2 \sim (1-2\beta)i^{-2\beta} k^{2\beta-1} \sim (1-2\beta)k^{-1}$, meaning that near the edge $\bv$ behaves like a dense vector. Consider the two cases close to entering the refinement regime: for the same fractional increase in the data size, since many features become learnable when $\beta < 1/2$ due to its dense-like behavior, it yields a steeper descent before the refinement regime is reached (Fig.~\ref{fig:err-overlaps}).

\textbf{Scaling laws for SGD-trained student.} We now discuss the
empirical findings of Fig.~\ref{fig:sgd-student} for SGD-trained students under the same fixed-teacher setting as before with power-law readout of exponent $\beta > 1/2$, namely that there exist: (1) a critical (i.e., optimized) width beyond which performance under \textsc{Adam} degrades, with a scaling law in model size that holds up to this threshold, and (2) two distinct scaling laws in data size, corresponding again to feature learning and refinement regimes. 

\emph{Scaling with model size.} We start by discussing the scaling with model size when the student width $k_s$ lies below the critical width, which is itself smaller than $k$. In this case, the number of features that the student can learn is equal to its width $k_s$. This is related to the \textit{capacity hypothesis} in \cite{michaud2023quantization}.

On the other hand, we observe empirically a mechanism similar to that underlying \eqref{eq:err-asymp}: each learned feature contributes $\Theta(d/n)$ to the generalization error, while each non-learned feature $i$ contributes a term of order $\Theta(v_i^2)$. Accordingly, the errors contributed by learned and non-learned features are $\Theta ( k_s d/n)$ and $\Theta ( \sum_{j >  k_s} v_j^2)$ respectively. Since the student cannot learn more features than the BO student, we have $k_c \geq k_s$, leading to
\begin{align*}
    \sum_{j > k_s} v_j^2  \geq \sum_{j > k_c} v_j^2 \gtrsim k_c d/n \geq  k_s d/n
\end{align*}
where the second inequality holds in the feature-learning regime (Appendix~\ref{app:scaling}), i.e. the current setting. Thus the SGD-student's error is $\Theta (\sum_{i > k_s } v_i^2 )=
\Theta ( \sum_{i > k_s } i^{-2\beta} )
    = \Theta ( \int_{k_s }^k x^{-2 \beta} \, dx ) = \Theta(k_s^{1-2\beta})$, giving the power-law scaling in model size with exponent $1-2\beta$
(dashed line, Fig.~\ref{fig:sgd-student}, middle).

Note that the optimally-sized student has a lower error overall, but a flatter slope in $n$, than the full-width student in the feature-learning regime:
precisely because it avoids fitting non-learnable directions, it starts
from a lower error floor and gains less from each additional sample.

\emph{Scaling with data.} That the exponents for the error as a function of $n$ of the SGD-trained model match the Bayes-optimal ones in the two regimes (Fig.~\ref{fig:sgd-student}, right panel) can be explained heuristically as follows. Firstly, a student of width $k_s \gg k_c$ overfits non-learnable directions; one near $k_c$ avoids this, as confirmed by Fig.~\ref{fig:sgd-student} (right). We conjecture that in this non-overfitting regime,
\textsc{Adam} learns each feature as in the single-feature GLM of Section~\ref{sec:derivations}, but with a higher threshold: feature $i$ is learned if and only if $\snr_i$ exceeds an algorithmic threshold $\lambda_\sigma^{\rm algo}\ge\lambda_\sigma$,
with the inequality reflecting a possible \emph{computational-to-statistical gap}~\cite{barbier2025statistical}, i.e.\ the difference between what can be theoretically recovered from the data and what can be recovered efficiently by known algorithms. Such a gap is expected here: for $\sigma=\He_3$, the GLM in the high-noise regime of Section~\ref{sec:derivations} is information-theoretically equivalent to a tensor PCA problem of order three, which has a computational-to-statistical gap~\cite{montanari2014statistical}. If $\lambda_\sigma^{\rm algo}$ and $\lambda_\sigma$ differ only by a constant factor, the scaling-law derivation above gives the same exponents. This conjecture is supported by numerics, and by the GLM itself, where the state evolution of approximate message passing (AMP) shows that AMP behaves as the Bayes-optimal estimator, but with higher thresholds~\cite{Montanari_SE,barbier2019optimal}. We do not prove it for \textsc{Adam} or for the present network.

We note one deviation from the predicted $n^{-1}$ slope at large $n$
in the refinement regime: roughly $5\%$ of the weakest features
remain non-learnable under \textsc{Adam}, causing the empirical
exponent to deteriorate. This residual gap in Fig.~\ref{fig:sgd-student} (right) may be due to the algorithmic threshold, or to our training procedure (number of epochs, learning rate, normalization of the features); our experiments do not distinguish between the two.

\paragraph{Extensive-width, near-interpolation regime.}

We also analyzed the linear-width, near-interpolation scaling regime $k \sim \gamma d$,
$n \sim \alpha d^2$ ($\alpha,\gamma>0$ fixed), where $k$, $d$,
and $n$ diverge \emph{together} -- a regime with rich feature-learning
phenomenology that attracts strong theoretical interest~\cite{maillard2024bayes, barbier2025statistical,
camilli2025information, defilippis2024dimension} connected to
compute-optimal scaling in large language
models~\cite{hoffmann2022training}.
The table below (Appendix~\ref{app:d-asymp}) reports the
asymptotic $\varepsilon^{\rm BO}$ and effective width $k_c$ for readout distributions of increasing sparsity in this regime. The $\Theta(\cdot)$ rates are in $d$ as $d\to\infty$ at fixed
$(\alpha,\gamma)$.

\renewcommand{\arraystretch}{1.3}
\begin{table*}[h]
\centering
\begin{tabular}{c|c|c}
\textbf{Readout distribution} & \textbf{Bayes-optimal error }$\varepsilon^{\rm BO}$ & \textbf{Effective width} $k_c$ \\ \hline
Dense & $\Theta(1)$ & $\Theta(d)$  \\ \hline 
Power-law, $\beta < 1/2$ & $\Theta(1)$ & $\Theta(d)$  \\ \hline
Power-law, $\beta = 1/2$ & $\Theta(\frac{\ln \ln d}{\ln d})$ & $\Theta ( \frac{d}{\ln d} ) $  \\ \hline 
Power-law, $\beta > 1/2$ & $\Theta(d^{\frac{1}{2\beta}-1})$ & $\Theta(d^{\frac{1}{2 \beta}})$ \\ \hline
Exponentially decaying & $\Theta(\frac{\ln d}{d})$ & $\Theta(\ln d)$ \\ \hline
Ultra sparse & $\Theta(\frac{1}{d})$ & $\Theta(1)$
\end{tabular}
\end{table*}
Sparsity has a clear impact on the scaling of the error and effective width: delocalized readouts ($\beta \le 1/2$ or dense) saturate at
$\varepsilon^{\rm BO} =\Theta(1)$ with $k_c = \Theta(d)$
learnable features, while sparser readouts allow decay down to
$\Theta(1/d)$ due to a finite effective width for ultra-sparse ones. For dense readouts we independently recover
the result of~\cite{barbier2025statistical} in Appendix~\ref{app:exact}; for power-law readouts we also characterize the limiting overlap profile.

\section{Derivation of the main results}\label{sec:derivations}
Our main results rely on a crucial statistical equivalence between the learning problem in the neural network model from the perspective of a single feature and a simpler ``high-sample rate generalized linear model''; it is derived in Appendix~\ref{app:glm}.
\begin{result}\label{res:glm-highnoise}
Let $\sigma$ be a non-constant, non-even, differentiable function such that $ \sigma, \sigma'$ are square-integrable w.r.t. the standard Gaussian measure on $\R$ (we do not need here to restrict to the case $\mu_2 = 0$). Let $n, d, \Delta$ be in the regime $n \gg d \gg 1$, where $\Delta$ is bounded away from zero and may depend on $d$, and define $\lambda := \frac{n}{d \Delta}$ (which may thus also depend on $d$). Consider the problem of inferring the signal $\bw^0 \sim \mathcal{U}(\mathbb{S}^{d-1})$ from the observations $\bY = (y_\mu)_{\mu=1}^n$ given by:
\begin{align}
    y_\mu = \sigma\!\left(\bw^0 \cdot \bx_\mu\right) + \sqrt{\Delta}\, z_\mu, \quad \mu \in [n],
\end{align}
 with inputs $\bx_\mu \sim \mathcal N(0, \bI_d) $ and i.i.d. noise variables $z_\mu \sim \mathcal N(0,1)$. Let $\bw^1$ be a sample from the posterior $P(\bw \mid \bY)$. Then, recalling definitions \eqref{defs}, we have
\begin{align} \label{eq:concentration-glm}
\El{ (\bw^0 \cdot \bw^1 - m_\sigma(\lambda) )^2 } = o_d(1). 
\end{align}
\end{result}

In particular, when the information exponent of $\sigma$ is at least $3$ (Appendix~\ref{app:hermite}), the model exhibits a sharp phase transition at $\lambda_\sigma$:  $m_\sigma(\lambda) = 0$ if $\lambda < \lambda_{\sigma}$, whereas $m_\sigma(\lambda) > 0$ if $\lambda > \lambda_\sigma$ (Appendix \ref{app:msig}). Result~\ref{res:glm-highnoise} can also be recovered by expanding the known formula for the Bayes-optimal error of GLMs~\cite{barbier2019optimal} in the high-sample-rate, high-noise limit, slightly beyond the regime where it is rigorously established.

We now provide the main steps of our heuristic derivation of Result~\ref{res:main}  based on the leave-one-out argument and Gaussian assumptions.

\textbf{Feature overlaps and Bayes-optimal error.} 
Recall that $\langle \,\cdot\, \rangle$ denotes the posterior mean w.r.t. $P(\bW \mid \mathcal D)$, and $\bW^1$ is a posterior sample. Fix $i \in [k]$. The leave-one-out approach relies on isolating the contribution of feature $\bw_i^0$ in the model \eqref{eq:model}:
\begin{equation}\label{eq:rewrite}
y_\mu^{\rm out} = v_i\,\sigma\!\left(\bw_i^0\cdot\bx_\mu\right)+R_{\mu i}(\bW^0)+\sqrt{\Delta}\,z_\mu,
\qquad
R_{\mu i}(\bW):=\sum_{j\neq i}v_j\,\sigma\!\left(\bw_j\cdot\bx_\mu\right).
\end{equation}
For a posterior sample $\bW$, decompose $R_{\mu i}=R_{\mu i}(\bW)$ into its posterior mean and a posterior fluctuation $\xi_{\mu i}$ that acts as an effective noise:
$R_{\mu i} = \langle R_{\mu i} \rangle + \xi_{\mu i}$, with $\xi_{\mu i}:=R_{\mu i} - \langle R_{\mu i} \rangle$. The posterior mean $\langle R_{\mu i}\rangle$ is a deterministic function of the dataset $\mathcal D$ that can be absorbed into the observation $y_\mu^{\rm out}\to y_\mu^{\rm out} - \langle R_{\mu i} \rangle$ without changing the posterior. We thus focus on the effective noise, whose variance is 
    $\langle (R_{\mu i} - \langle R_{\mu i} \rangle )^2 \rangle \sim \sum_{j \neq i} v_j^2(g_\sigma(1)-g_\sigma( q_j )) := \varepsilon_{-i}$. This together with the relation below,
\begin{align}\label{eq:BO_main}
    \varepsilon^{\rm BO} \sim \sum_{i = 1}^k v_i^2 (g_\sigma(1) - g_\sigma(q_i)) := \varepsilon.
\end{align}
are derived for a special case -- and assumed to hold in general -- in Appendix~\ref{app:main-res}.
\emph{We make the additional assumption that, for fixed $i$, the effective noises $(\xi_{\mu i})_{\mu\in A}$ converge in law to independent Gaussian variables, for any finite set $A\subset[n]$ of samples} -- this central limit theorem heuristic, standard in the cavity method \cite{talagrand2010mean}, is supported a posteriori by the excellent match between our theory and experiments, as well as via the analytical match with the equations of \cite{barbier2025statistical} in the dense readout case. Heuristically, the posterior depends only weakly on finitely many samples $(\bx_\mu)_{\mu\in A}$, so each $\xi_{\mu i}$ behaves as a sum of $k-1$ weakly dependent terms, none of which dominates; it is then Gaussian by the central limit theorem, and the independence across $\mu$ follows from that of the inputs. Making this rigorous requires controlling the correlations induced by the posterior, which we leave open. Combining the effective noise with the original label noise yields a single effective Gaussian noise of variance $\Delta+\varepsilon_{-i}$, so that~\eqref{eq:rewrite} (after subtracting the deterministic $\langle R_{\mu i}\rangle$) takes the form
\begin{equation}\label{eq:effective-glm}
\text{observation}_\mu = v_i \sigma\!\left( \bw_i^0\cdot \bx_\mu \right) + \sqrt{\Delta + {\varepsilon_{-i}}}\, \text{noise}_\mu.
\end{equation}
For $v_i>0$, dividing~\eqref{eq:effective-glm} by $v_i$ gives a GLM with noise variance $(\Delta+\varepsilon_{-i})/v_i^2\geq\Delta$, which is in the regime of Result~\ref{res:glm-highnoise}. Applying this result gives
\begin{align}
    \bw^0_i \cdot \bw^1_i \simeq m_\sigma \Big( \frac{nv_i^2/d}{\Delta + \varepsilon_{-i} } \Big) \simeq m_\sigma \Big( \frac{nv_i^2/d}{\Delta + \varepsilon } \Big)
\end{align}
where in the last step we used $\varepsilon_{-i} = \varepsilon+o(1)$, away from transition points where $m_\sigma$ is continuous: the difference $v_i^2(g_\sigma(1)-g_\sigma(q_i))$ is $o(1)$ if $v_i\to0$, and $O(d/n)$ otherwise, since feature $i$ is then learned as $n\gg d$. Together with the equation for $\varepsilon$, this closes the system of fixed-point equations~\eqref{eq:fpe}.

\textbf{Asymptotic Bayes-optimal error.} From \eqref{eq:BO_main}, and using $q_i=0$ for $i>k_c$, we have 
\begin{align*}
    \varepsilon^{\rm BO} \sim \sum_{i \leq k_c} v_i^2 (g_\sigma(1) - g_\sigma(q_i)) + \sum_{k_c<i\leq k} v_i^2 g_\sigma(1).
\end{align*}
The second sum, which is the error due to non-learnable features, is $\Theta( \sum_{k_c<i\leq k} v_i^2 )$. On the other hand, from the fixed-point equations \eqref{eq:fpe}, we have  $q_i = m_\sigma({\rm snr}_i)$, where recall that ${\rm snr}_i= (n/d)v_i^2/(\Delta + \varepsilon)$. For any $i \leq k_c$, feature $i$ is learnable, so ${\rm snr}_i > \lambda_\sigma$. The first sum becomes
\begin{align*}
    (\Delta+\varepsilon) \frac{d}{n} \sum_{i \leq k_c}  {\rm snr}_i\, \big(g_\sigma(1) - g_\sigma \circ m_\sigma({\rm snr}_i) \big),
\end{align*}
which is $\Theta(k_cd/n)$, since $\Delta + \varepsilon = \Theta(1)$ and there exist absolute constants $c_1, c_2 >0$ such that $x(g_\sigma(1)-g_\sigma \circ m_\sigma(x)) \in [c_1, c_2]$ for all $x> \lambda_\sigma$ (Appendix \ref{app:msig}).

From the derivation, each learnable feature contributes $\Theta(d/n)$ to the BO error, while each non-learnable feature $i$ contributes $\Theta(v_i^2)$.

\section{Conclusion and limitations} \label{sec:conclusion}
We studied the information-theoretic limits of learning one- and two-hidden-layer teacher networks in the extensive-width regime
$k=\Theta(d)$, a setting that captures large-but-finite-width
networks and sits strictly between the multi-index and NTK/infinite-width extremes. We provided a heuristic leave-one-out argument yielding closed
fixed-point equations for the individual feature overlaps $(q_i)$
and the Bayes-optimal error $\varepsilon^{\rm BO}$. A single
activation-dependent SNR threshold $\lambda_\sigma$ governs feature
learnability: as $n$ grows, features cross it one by one, producing
a staircase of sharp phase transitions in the learning curve. For
a power-law hierarchy of features with exponent $\beta>1/2$, the resulting effective width
$k_c\le k$ unifies the two empirically observed scaling regimes via
$\varepsilon^{\rm BO}=\Theta(k_c d/n)$: a feature-learning regime in
which $k_c$ grows with $n$, and a refinement regime in which $k_c=k$
is saturated and the parametric rate $n^{-1}$ is recovered. An
\textsc{Adam}-trained student near $k_c$ tracks these Bayes-optimal
exponents up to a small residual gap on the weakest features.

The main limitations of our work are that derivations rest on unproved cavity assumptions
-- Gaussian cavity fields and overlap concentration -- and the
restriction $\mu_2=0$. The discontinuity of the transitions relies on this restriction (Fig.~\ref{fig:glm}), although we expect the scaling laws to remain valid for $\mu_2\neq0$, as suggested by~\cite{defilippis2026optimal,barbier2025statistical}. Tackling general activations in the extensive-width regime $k=\Theta(d)$ is a technical challenge that has been considered only recently  \cite{barbier2025statistical}. Closing the $\mu_2$
gap~\cite{defilippis2026optimal,maillard2024bayes}, proving the
leave-one-out reduction, extending to multi-layer
architectures~\cite{barbier2025statistical}, and characterizing
compute-optimal scaling~\cite{hoffmann2022training} are natural
next steps. Finally, the inputs are isotropic Gaussian; for a general covariance, we expect tractable modified equations, as in~\cite{barbier2025statistical}.

\section*{Acknowledgments}
J.B. and M.-T.N. were funded by the European Union (ERC, CHORAL, project number 101039794). Views and opinions expressed are however those of the authors only and do not necessarily reflect those of the European Union or the European Research Council. Neither the European Union nor the granting authority can be held responsible for them.
 
\bibliography{power_law}
\bibliographystyle{plain}

\newpage
\appendix
\onecolumn

\section{Hermite expansion and information exponent}\label{app:hermite}
The $\ell$-th Hermite polynomial is defined as
\begin{align}\label{eq:hermite-def}
\He_\ell(x) := (-1)^\ell e^{x^2/2} \frac{d^\ell}{dx^\ell} e^{-x^2/2}, 
\quad \ell \ge 0.
\end{align}
$\He_\ell$ is a polynomial of degree $\ell$ with leading coefficient $1$. The first few examples are $\He_0(x)=1, \He_1(x)=x, \He_2(x)=x^2-1$, $\He_3(x)=x^3-3x$. Any $\sigma \in L^2(\mu)$ ($\mu$ is the standard Gaussian measure) admits the \emph{Hermite expansion}
\begin{equation}\label{eq:hermite-expansion}
\sigma(z) = \sum_{\ell=0}^\infty \frac{\mu_\ell}{\ell!} \He_\ell(z),
\end{equation}
in which the $\ell$-th \emph{Hermite coefficient} of $\sigma$ is given by $\mu_\ell = \E_{z \sim \mathcal N(0, 1)}\sigma(z) \He_\ell(z)$.

An important quantity is the \emph{information exponent} of $\sigma$ \cite{arous2021online}, defined as
\begin{equation}\label{eq:info-exponent}
k_\star(\sigma) := \min \{ \ell \ge 1 : \mu_\ell \neq 0 \}.
\end{equation}

\section{Heuristic derivation of Result \ref{res:glm-highnoise}} \label{app:glm}
We provide here a heuristic derivation of Result \ref{res:glm-highnoise}. We start by deriving the fixed-point equation for $m_\sigma(\lambda)$. Assuming $\bw^0 \cdot \bw^1 \simeq q$ where $\bw=\bw^1$ is a posterior sample, we will show that $q/(1-q) = \lambda\, g'_\sigma(q)$. Fix $i \in [d]$ and focus on the estimation of $w^0_i$
\begin{align*}
    y_\mu = \sigma \left( w^0_i x_{\mu i} + A_{\mu i} \right) + \sqrt{ \Delta} \, z_\mu, \quad \mu \in [n],
\end{align*}
where $A_{\mu i} = \sum_{j \neq i} w_j^0 x_{\mu j}$. Since $w^0_i x_{\mu i}$ is small compared to $A_{\mu i}$, we can expand the model around $A_{\mu i}$ and obtain
\begin{align*}
    y_\mu &\simeq \sigma(A_{\mu i}) + \sigma'(A_{\mu i}) w^0_i x_{\mu i} + \sqrt{\Delta} z_\mu, \quad \mu \in [n].
\end{align*}
We claim that $\sigma(A_{\mu i})$ and $\sigma'(A_{\mu i})$ can be replaced by their posterior averages $\langle \sigma(A_{\mu i})\rangle$ and $\langle \sigma'(A_{\mu i})\rangle$: the fluctuations around the average contribute additional noise, negligible against the channel noise with variance $\Delta$. This is because, if $\Delta = \Theta(1)$, then since $n \gg d$, the overlap for $\bw^0$ is almost $1$, making the fluctuations of $\sigma(A_{\mu i})$ and $\sigma'(A_{\mu i})$ of order $o(1)$, and thus negligible compared to $\Delta$. On the other hand, if $\Delta \gg 1$, since the fluctuations of $\sigma(A_{\mu i})$ and $\sigma'(A_{\mu i})$ are bounded, they remain negligible compared to $\Delta$.

Subtracting $\langle \sigma(A_{\mu i})\rangle$ yields the effective scalar model
\begin{align*}
    \widetilde y_\mu = \langle \sigma'(A_{\mu i})\rangle\, w^0_i\, x_{\mu i} + \sqrt{\Delta}\, \widetilde z_\mu, \quad \mu \in [n],
\end{align*}
which is equivalent to a single Gaussian channel for $w^0_i$ with effective signal-to-noise ratio
\begin{align*}
    r_i := \frac{1}{\Delta} \sum_{\mu=1}^n \langle \sigma'(A_{\mu i}) \rangle^2 x_{\mu i}^2.
\end{align*}
Suppose that $r_i$ concentrates around its mean. Using $\E[x_{\mu i}^2]=1$ together with the approximate independence between $x_{\mu i}^2$ and $\langle \sigma'(A_{\mu i}) \rangle$, then replacing $A_{\mu i}$ by $\bw \cdot \bx_\mu$ inside the bracket (an $O(d^{-1/2})$ correction), we obtain
\begin{align*}
    \E[r_i] \sim \frac{n}{\Delta}\, \E\!\left[\langle \sigma'(\bw \cdot \bx_\mu)\rangle^2\right] = \lambda d\, \E\!\left[\langle \sigma'(\bw \cdot \bx_\mu)\rangle^2\right].
\end{align*}

\noindent\emph{Computing $\E[\langle\sigma'(\bw\cdot\bx_\mu)\rangle^2]$.} 
The posterior depends on $\bx_\mu$, so Informal statement~\ref{lem:talagrand} below does not apply directly. However, removing observation $\mu$ gives a cavity posterior that is independent of $\bx_\mu$, and since $n\gg d$, we expect that adding this single observation back changes the averages below only by $o(1)$. Applying Informal statement~\ref{lem:talagrand} to the cavity posterior gives
\begin{align*}
    \bw \cdot \bx_\mu \stackrel{(d)}{\simeq} \langle \bw \cdot \bx_\mu \rangle + \sqrt{1-q}\, \xi, \qquad \xi \sim \mathcal{N}(0,1).
\end{align*}
Since $\bx_\mu \sim \mathcal{N}(0, \bI_d)$, the average term is Gaussian with variance $\|\langle \bw \rangle\|^2 \simeq q$. We obtain
\begin{align*}
    \bw \cdot \bx_\mu \;\stackrel{(d)}{\simeq}\; \sqrt{q}\,\xi_0 + \sqrt{1-q}\,\xi_1, \qquad \xi_0,\,\xi_1 \iid \mathcal{N}(0,1),
\end{align*}
where $\xi_0$ is measurable with respect to the data and $\xi_1$ captures the posterior uncertainty. We obtain
\begin{align*}
    \E[\langle\sigma'(\bw\cdot\bx_\mu)\rangle^2] \simeq  \E_{\xi_0}  (\E_{\xi_1} \sigma'(\sqrt q \xi_0 + \sqrt{1-q} \xi_1))^2 = g_{\sigma'} (q) = g'_\sigma(q)
\end{align*}
where the second equality follows from Proposition \ref{prop:gm}. 

Summarizing the previous arguments, for each $i$, the posterior of $w_i$ is asymptotically that of a Gaussian channel with signal $w_i^0$ and SNR $\lambda d g'_\sigma(q)$. Assuming independence across individual channels, the posterior of $\bw$ is asymptotically that of a Gaussian channel with signal $\bw^0$ and the same SNR. This channel admits a simple SNR-overlap relation, $\snr/d = q/(1-q)$, yielding the fixed-point equation
\begin{align*}
    q/(1-q) = \lambda g'_\sigma(q).
\end{align*}
This equation may have multiple solutions and thus does not uniquely determine $q$.  A unique characterization of the overlap is instead achieved by maximizing a large deviation rate function $I(q)$, defined informally via
\begin{align*}
    \frac{1}{d} \ln \mathbb P( \bw^0 \cdot \bw^1 \simeq q) \simeq I(q).
\end{align*}
The fixed-point equation is then simply $I'(q) = 0$. Integrating both sides of $q/(1-q) - \lambda g'_\sigma(q) = 0$ suggests that, up to an affine transform, the rate function is
\begin{align}
    \lambda g_\sigma(q) + q + \ln(1-q), \label{eq:potential}
\end{align}
which yields the unique characterization of the overlap given by $m_\sigma(\lambda)$ via its argmax. Importantly, this integration argument relies on the specific form of the fixed-point equation: the left-hand side $q/(1-q)$ has a clear interpretation as the SNR of a Gaussian channel expressed in terms of the overlap, and it is this structure that makes this work. Note that integrating arbitrary equivalent form of the fixed-point equation, such as $q = (1-q)\lambda g'_\sigma(q)$, does not recover the correct rate function. The resulting function \eqref{eq:potential} is called a ``replica symmetric potential'' in physics parlance, and can also be guessed by taking the joint high-sample-rate, high-noise limit, i.e. $\alpha=n/d\to\infty$ and $\Delta\to\infty$ with $\alpha/\Delta\to\lambda$ fixed, in the replica symmetric potential for the standard generalized linear model \cite{barbier2019optimal}.

\begin{informal}\label{lem:talagrand}
(by Talagrand \cite{talagrand2010mean}, Chapter~1.5)
Consider a probability measure $\mu=\mu_d$ in $\R^d$. Denote $\langle\, \cdot \, \rangle$ the average with respect to $\mu$. Suppose there exist $q$ and $\rho$ such that for large $d$:
\begin{align}
    \langle (\bw^1 \cdot \bw^2 - q)^2 \rangle = o_d(1) \label{talagrand-1} \\
    \langle (\|\bw\|^2 - \rho)^2 \rangle = o_d(1) \label{talagrand-2}.
\end{align}
Let $\bx \sim \mathcal N(0, \bI_d)$ independent of $\mu$. Then
\begin{align*}
    \bw \cdot \bx \stackrel{(d)}{\simeq} \langle \bw \cdot \bx \rangle + \sqrt{1-q}\, \xi, \qquad \xi \sim \mathcal{N}(0,1).
\end{align*}
\end{informal}
We apply this statement, via the leave-one-out argument above, when $\mu$ is the posterior distribution in the current model. The conditions \eqref{talagrand-1} and \eqref{talagrand-2} are satisfied by high-dimensional posterior distributions in more general settings of optimal Bayesian inference, as shown in \cite{barbier2022strong}.

\section{Estimating the linear part via linear regression} \label{app:linear-part}
We provide here the justification that the activation $\sigma(x)$ can be replaced by $\sigma(x) - \mu_0 -\mu_1 x$ without changing the asymptotic behavior of the posterior. 

Without loss of generality, we can assume that $\mu_0=0$, since this is a constant that can be subtracted from $y_\mu^{\rm out}$. We will show that $\bs^0 = \bW^{0 \top} \bv$ can be estimated with vanishing error with the following estimator
\begin{align} \label{eq:linear-part-estimation}
\hat \bs &= \frac{1}{n \mu_1} \bX \by.
\end{align}

Writing $\by$ as
\begin{align*}
    \by = \mu_1 \bX^\top \bs^0 + \bxi, \quad \bxi = \bv^\top \tilde \sigma(\bW^0 \bX) + \sqrt{\Delta} \bz,
\end{align*}
we have
\begin{align*}
    \hat \bs &= \frac{1}{n} \bX \bX^\top \bs^0 + \frac{1}{n \mu_1} \bX \bxi.
\end{align*}
Since $n \gg d$, we have $\bX \bX^\top/n \simeq \bI_d$ which is a basic property of Wishart matrices. Therefore 
\begin{align}\label{signal-part}
    \frac{1}{n} \bX \bX^\top \bs^0  \simeq \bs^0.
\end{align}
On the other hand,
\begin{align*}
    \frac{1}{n}\bX \bxi &= \frac{1}{n} \sum_{\mu=1}^n \xi_\mu \bx_\mu \\
    & \simeq \E_{\bx, z}[(\bv^\top \tilde \sigma(\bW^0 \bx) + \sqrt{\Delta} z) \bx], \quad \text{since } n \gg d  \\
    &= \E_{\bx}[(\bv^\top \tilde \sigma(\bW^0 \bx))\bx], \quad \text{since } \E[\bx z] = 0 \\
    &= \sum_{i=1}^k v_i \E[\tilde \sigma(\bw^0_i \cdot \bx) \bx].
\end{align*}
Recall the following basic property of Hermite polynomials, where
\begin{align*}
    \E[\He_{\ell}(z_1) \He_{\ell'}(z_2)] = 0
\end{align*}
for all $\ell \neq \ell'$ and $(z_1, z_2)$ jointly Gaussian, each with variance $1$. This implies $\E[\tilde \sigma(\bw^0_i \cdot \bx) x_j] = 0$, or $\E[\tilde \sigma(\bw^0_i \cdot \bx) \bx] = 0$. Therefore
\begin{align}\label{noise-part}
    \frac{1}{n} \bX \bxi \simeq 0.
\end{align}
Combining \eqref{signal-part} and \eqref{noise-part}, $\hat \bs\simeq  \bs^0$, so the linear contribution can be perfectly estimated and removed; the residual posterior on $\bW_0$
 then matches that of the model with $\tilde \sigma$ in place of $\sigma$.

\section[Properties of]{Properties of $g_\sigma$ and $m_\sigma$} \label{app:msig}
\begin{proposition} \label{prop:gm}
Let $\sigma$ be a non-constant, differentiable function and assume $\sigma, \sigma' \in L^2(\mu)$, where $\mu=\mathcal N(0,1)$. Then the functions $g_\sigma$ and $m_\sigma$ defined in \eqref{defs0}, \eqref{defs} satisfy the following properties:
\begin{enumerate}
    \item $g_\sigma$ has the expansion 
    \begin{align} \label{eq:gsigma-expansion}
        g_\sigma(x) = \sum_{\ell=0}^\infty \frac{\mu_\ell^2}{\ell!} x^\ell
    \end{align}
    where $\mu_\ell = \E_{z \sim \mathcal N(0,1)}[\sigma(z) \He_\ell(z)]$ is the $\ell$-th Hermite coefficient of $\sigma$. As a consequence, $g_\sigma$ is strictly increasing in $[0,1]$.
    \item $m_\sigma$ is non-decreasing, $1-m_\sigma(x) \sim c/x $ as $x \ra \infty$ for some constant $c>0$. Consequently, there exists $c_1, c_2 > 0$ such that $x(g_\sigma(1)-g_\sigma \circ m_\sigma(x)) \in [c_1, c_2]$ for all $x> \lambda_\sigma$
    \item If $k_\star(\sigma) \geq 3$, there exists $\lambda_\sigma > 0$, $q_\sigma>0$ such that $m_\sigma(\lambda) = 0$ if $\lambda < \lambda_\sigma$ and $\lim_{\lambda \ra \lambda_\sigma^+} m_\sigma(\lambda) = q_\sigma$.
    \item $g'_\sigma = g_{\sigma'}$.
\end{enumerate}
\end{proposition}

\begin{proof}
1. Recall \emph{Mehler's formula}: for $(z_1, z_2)$ a centered Gaussian vector with unit variances and correlation $x \in [-1,1]$,
\begin{equation}\label{eq:mehler}
\E[\He_\ell(z_1) \He_m(z_2)] = \delta_{\ell m}\, \ell!\, x^\ell, \qquad \ell, m \ge 0.
\end{equation}
Writing $\sigma = \sum_{\ell \ge 0} \frac{\mu_\ell}{\ell!} \He_\ell$ in $L^2(\mu)$ and using the orthogonality \eqref{eq:mehler}, we obtain, for every $x \in [-1,1]$,
\begin{align*}
g_\sigma(x) = \E[\sigma(z_1)\sigma(z_2)] 
= \sum_{\ell, m \ge 0} \frac{\mu_\ell \mu_m}{\ell!\, m!} \E[\He_\ell(z_1)\He_m(z_2)]
= \sum_{\ell \ge 0} \frac{\mu_\ell^2}{\ell!}\, x^\ell.
\end{align*}
Justification of the interchange: we have
\begin{align*}
\sum_{\ell, m} \frac{|\mu_\ell \mu_m|}{\ell!\, m!}\, |\E[\He_\ell(z_1)\He_m(z_2)]| = \sum_{\ell} \frac{\mu_\ell^2}{\ell!}\, |x|^\ell \le \sum_\ell \frac{\mu_\ell^2}{\ell!} = \|\sigma \|^2_{L^2(\mu)} < \infty
\end{align*}
 for $|x| \le 1$, so Fubini applies. The series \eqref{eq:gsigma-expansion} therefore converges absolutely on $[-1,1]$, and in particular at $x=1$.

2. Recall that
\begin{align*}
    m_\sigma(\lambda) = \arg \max_{q \in [0,1)} F(\lambda, q), \quad F(\lambda, q) = \lambda g_\sigma(q) + q + \ln(1-q).
\end{align*}
\textit{Uniqueness.} The function $V(\lambda)=\max_q F(\lambda,q)$ is convex, as a maximum of affine functions of $\lambda$, so it is differentiable except at countably many points. At a differentiability point, every maximizer $q$ satisfies $V'(\lambda)=g_\sigma(q)$. Since $g_\sigma$ is strictly increasing on $[0,1]$ by~\eqref{eq:gsigma-expansion}, the maximizer is unique. Hence $m_\sigma(\lambda)$ is well defined except at countably many transition points.

\textit{Monotonicity.} Consider $\lambda_1 > \lambda_2 \geq 0$. Let $q_1 = m_\sigma(\lambda_1)$, $q_2 = m_\sigma(\lambda_2)$. By optimality of $q_1$ and $q_2$,
\begin{align*}
    F(\lambda_1, q_1) \geq F(\lambda_1, q_2), \\
    F(\lambda_2, q_2) \geq F(\lambda_2, q_1).
\end{align*}
Adding the two inequalities and simplifying,
\begin{align*}
    (\lambda_1 - \lambda_2)\big(g_\sigma(q_1) - g_\sigma(q_2)\big) \ge 0.
\end{align*}
Since $\lambda_1 > \lambda_2$, this implies $g_\sigma(q_1) \geq g_\sigma(q_2)$. Since $g_\sigma$ is strictly increasing (Part~1), we have $q_1 \geq q_2$. We conclude that $m_\sigma$ is non-decreasing in $[0, \infty)$.

\textit{Asymptotic behavior.} Let $q = m_\sigma(\lambda)$. Then $q$ satisfies $q/(1-q) = \lambda g'_\sigma(q)$. When $\lambda \ra \infty$, we have $q \ra 1$. From $q = (1-q) \lambda g'_\sigma(q)$, as $\lambda \ra \infty$, we have $1 \sim (1-q) \lambda g_\sigma'(1)$, which gives $m_\sigma(\lambda) = q \sim 1 - c/\lambda$, where $c = 1/g'_\sigma(1)$.

\textit{Boundedness of $\kappa(x):= x(g_\sigma(1)-g_\sigma \circ m_\sigma(x))$ on $ (\lambda_\sigma, \infty)$}. From the asymptotic of $m_\sigma$, we have $\kappa(x) \ra 1$ as $x \ra \infty$. Fix $\epsilon>0$. There exists $M>0$ such that $ \kappa(x) \in [1- \epsilon, 1+\epsilon] $ for all $x \geq M$. For $x \in (\lambda_\sigma, M]$, $\kappa(x)$ is bounded above by $M g_\sigma(1)$ and bounded below by $\lambda_\sigma (g_\sigma(1) - g_\sigma \circ m_\sigma(M)) > 0$.

3. We first observe that $q=0$ is always a local maximum of $F(\lambda, q)$ for every $\lambda \ge 0$. Indeed, we have
\begin{align*}
    \partial_q^2 F(\lambda, q) &= \lambda g''_\sigma(q) - \frac{1}{(1-q)^2},
\end{align*}
which gives $\partial_{q=0}^2 F(\lambda, q) =  - 1 < 0$, since $g''_\sigma(0)=0$ as $k_\star(\sigma) \geq 3$. Next, $q=0$ is the global maximum if and only if
\begin{align*}
    \lambda g_\sigma(q) + q + \ln(1-q) < 0, \quad \forall q \in (0,1),
\end{align*}
which is equivalent to
\begin{align}
    \lambda < \lambda_\sigma := \inf_{q \in (0, 1)} h(q), \quad h(q) = \frac{-q - \ln(1-q)}{g_\sigma(q)}.
\end{align}
Let $q_\sigma$ be the largest minimizer of $h(q)$ on $(0,1)$. Since $h(q)>0$ for all $q \in (0,1)$, $\lim_{q \ra 0^+} h(q) = \lim_{q \ra 1^-} h(q) = \infty$,  it follows that $q_\sigma, \lambda_\sigma$ are well-defined and $\lambda_\sigma \in (0, \infty)$. We now show that $\lim_{\lambda \ra \lambda_\sigma^+} m_\sigma(\lambda) = q_\sigma$. By definition of $\lambda_\sigma$, we have $F(\lambda_\sigma,q)\leq 0$ for all $q\in[0,1)$, with equality at $q=0$ and $q=q_\sigma$. Let $\lambda>\lambda_\sigma$ and $q_\lambda=m_\sigma(\lambda)$. Since $q_\lambda$ maximizes $F(\lambda,\cdot)$,
\begin{align*}
0 &\geq F(\lambda_\sigma,q_\lambda) \\
&= F(\lambda,q_\lambda)-(\lambda-\lambda_\sigma)\,g_\sigma(q_\lambda) \\
&\geq F(\lambda,q_\sigma)-(\lambda-\lambda_\sigma)\,g_\sigma(q_\lambda), \qquad \text{since $q_\lambda$ maximizes $F(\lambda,\cdot)$} \\
&= F(\lambda_\sigma,q_\sigma)+(\lambda-\lambda_\sigma)\big(g_\sigma(q_\sigma)-g_\sigma(q_\lambda)\big) \\
&= (\lambda-\lambda_\sigma)\big(g_\sigma(q_\sigma)-g_\sigma(q_\lambda)\big), \qquad \text{since $F(\lambda_\sigma,q_\sigma)=0$}.
\end{align*}
Hence $g_\sigma(q_\lambda)\geq g_\sigma(q_\sigma)$, so $q_\lambda\geq q_\sigma$. Moreover, for $\lambda\in(\lambda_\sigma,\lambda_\sigma+1]$, we have $F(\lambda,q_\lambda)\geq F(\lambda,0)=0$ and $F(\lambda,q)\leq(\lambda_\sigma+1)g_\sigma(1)+q+\ln(1-q)\to-\infty$ as $q\to1$, so $q_\lambda$ stays away from $1$. Let $\bar q$ be a limit point of $q_\lambda$ as $\lambda\downarrow\lambda_\sigma$. By continuity, $F(\lambda_\sigma,\bar q)\geq 0$, hence $F(\lambda_\sigma,\bar q)=0$, i.e.\ $\bar q$ minimizes $h$. Since $\bar q\geq q_\sigma$ and $q_\sigma$ is the largest minimizer of $h$, we get $\bar q=q_\sigma$.

4.  Expand
\[
\sigma(z)=\sum_{\ell\ge 0} a_\ell \He_\ell(z), \qquad a_\ell = \mu_\ell / \ell!.
\]
Using $\He_\ell'(z)=\ell\,\He_{\ell-1}(z)$ and the assumption $\sigma' \in L^2(\mu)$ (so the term-by-term differentiation is justified in $L^2$), we obtain
\[
\sigma'(z)=\sum_{\ell\ge 1} a_\ell \ell\, \He_{\ell-1}(z).
\]

Mehler's formula gives
\[
\E[\He_\ell(z_1)\He_m(z_2)] = \delta_{\ell m}\, \ell!\, x^\ell \quad \text{for}\quad (z_1,z_2)\sim\mathcal N\!\left(0,\begin{bmatrix}1 & x \\ x & 1\end{bmatrix}\right).
\]
Therefore,
\[
g_\sigma(x)=\E[\sigma(z_1)\sigma(z_2)]
= \sum_{\ell\ge 0} a_\ell^2\, \ell!\, x^\ell,
\]
so
\[
g'_\sigma(x)=\sum_{\ell\ge 1} a_\ell^2\, \ell!\, \ell\, x^{\ell-1}.
\]
On the other hand,
\[
g_{\sigma'}(x)
=\E[\sigma'(z_1)\sigma'(z_2)]
=\sum_{\ell\ge 1} a_\ell^2\, \ell^2\, (\ell-1)!\, x^{\ell-1}
=\sum_{\ell\ge 1} a_\ell^2\, \ell!\, \ell\, x^{\ell-1}.
\]
Hence $g'_\sigma(x)=g_{\sigma'}(x)$ for $x \in [0,1]$.
\end{proof}

 \section{Proof of feature weak coupling in the dense readout case} \label{app:concentration}
 Let us prove the feature weak coupling -- first identity in \eqref{concentration-2} -- for the special case $\sqrt{k}\bv = \1$; the proof relies on the exchangeability of the features in this case, and extending it to structured readouts requires new ideas. We have
 \begin{align*}
     \sum_{i,j=1}^k \E \langle (\bw_i^0 \cdot \bw^1_j)^2 \rangle &= \E \Big \langle \sum_{i,j=1}^k  \sum_{l,m=1}^d w^0_{il} w^1_{jl} w^0_{im} w^1_{jm} \Big \rangle \\
     &= \E \Big \langle \sum_{l,m}  [\bW^{0 \top} \bW^0]_{lm} [\bW^{1 \top} \bW^1]_{lm} \Big \rangle \\
     &= \E \langle \tr( \bW^{0 \top} \bW^0 \bW^{1 \top} \bW^1) \rangle \\
     &\leq \sqrt{ \E \langle \| \bW^{0 \top} \bW^0 \|_F^2 \rangle \E \langle \| \bW^{1 \top} \bW^1 \|_F^2 \rangle } \\
     &= \E \| \bW^{0 \top} \bW^0 \|_F^2 =  \E \| \bW^{0} \bW^{0 \top} \|_F^2 = k.
 \end{align*}
 This sum contains $k$ equal terms of order $O(1)$ and other $k(k-1)$ ``off-diagonal'' terms that are equal by statistical exchangeability. This implies that the off-diagonal terms are of order $O(k^{-1}) = O(d^{-1})$ as claimed.

\section{More details on the derivation of the main result} \label{app:main-res}

We fill in some details of the derivation of the main result.  First, we show that
\begin{align}\label{eBO}
\varepsilon^{\rm BO} \sim \sum_{i=1}^k v_i^2 \bigl(g_\sigma(1) - g_\sigma(q_i)\bigr)
\end{align}
holds in the dense readout, interpolation regime, while leaving a complete proof for future work.

Since the BO error is half the Gibbs error, we have
\begin{align*}
    \varepsilon^{\rm BO} &= \frac{1}{2} \E_{\mathcal D, \bW, \bW^1} \E_{\bx \sim \mathcal N(0, \bI_d)} [(y_1 - y_0)^2] 
\end{align*}
where
\begin{align*}
    y_0  = \bv^\top \sigma(\bW^0 \bx), \quad y_1  = \bv^\top \sigma(\bW^1 \bx)
\end{align*}
and $\bW^1$ is drawn from the posterior $P(\bW \mid \mathcal D)$. We have
\begin{align*}
    \E_{\bx} [(y_1 - y_0)^2] &= \E_{\bx}[y_0^2] + \E_\bx [y_1^2] - 2 \E[y_0 y_1] \\
    \E_{\bx}[y_1^2] = \E_{\bx}[y_0^2] &= \sum_{i,j=1}^k v_i v_j \E_{\bx} [\sigma(\bw^0_i \cdot \bx) \sigma(\bw_j^0 \cdot \bx)] 
    = \sum_{i,j=1}^k v_i v_j g_\sigma(\bw^0_i \cdot \bw^0_j) = g_\sigma(1) \\
    \E_{\bx} [y_0 y_1] &= \sum_{i,j=1}^k v_i v_j \E_{\bx} [\sigma(\bw^0_i \cdot \bx) \sigma(\bw_j^1 \cdot \bx)]
    = \sum_{i,j=1}^k v_i v_j g_\sigma(\bw^0_i\cdot \bw^1_j).
\end{align*}
Due to the assumed concentration of $\E_{\bx}[(y_0-y_1)^2]$ with respect to $(\mathcal D,\bW^0,\bW^1)$, we have
\begin{align*}
    \varepsilon^{\rm BO} \sim \frac{1}{2} \E_{\bx}[(y_0-y_1)^2] = \underbrace{\sum_{i=1}^k v_i^2 (g_\sigma(1) - g_\sigma(\bw^0_i \cdot \bw^1_i))}_{A} - \underbrace{\sum_{i \neq j} v_i v_j g_\sigma(\bw^0_i \cdot \bw^1_j)}_{B}
\end{align*}

It suffices to show that $B$ is negligible compared to $\varepsilon^{\rm BO}$, since this implies
\begin{align*}
    \varepsilon^{\rm BO} \sim A \sim \sum_{i=1}^k v_i^2 (g_\sigma(1)-g_\sigma(q_i))
\end{align*}
 
Since $g_\sigma(x) = O(x^3)$ when $k_\star(\sigma) \geq 3$, each $g_\sigma(\bw^0_i \cdot \bw^1_j)$, with $i \neq j$, is of order $d^{-3/2}$. Moreover, when $\bv$ is dense, the entries $v_i$ are of order $k^{-1/2}$, it follows that $B = O(k d^{-3/2}) = O(d^{-1/2})$. In the dense readout, interpolation regime, $\varepsilon^{\rm BO} = \Theta(1)$. We conclude $B \ll \varepsilon^{\rm BO}$.

The same argument gives the variance of $R_{\mu i}(\bW)$, under the leave-one-out and Gaussian assumptions of Section~\ref{sec:derivations}, since $R_{\mu i}$ only differs from $\bv^\top \sigma(\bW \bx_\mu)$ by a single term $v_i \sigma(\bw_i \cdot \bx_\mu)$.

For the general case, we give here some insights that may help to prove \eqref{eBO}. Let us denote $Q_{ij}:= \bw^0_i \cdot \bw^1_j$.
\begin{enumerate}
    \item $\sum_{i} Q_{ij}^2 = \sum_j Q_{ij}^2 = 1 $. This identity is useful for bounding $B$ when $n$ is very large, as in this regime the overlaps $Q_{ii} \simeq q_i$ are close to $1$, which suppresses the off-diagonal terms.
    \item Experimentally, we observe that $Q_{ij}$ for $i \neq j$ are \textit{asymptotically uncorrelated}, i.e.
    \begin{align*}
        \var(B) \lesssim \sum_{i \neq j} v_i^2 v_j^2 \var( Q_{ij}^3 ),
    \end{align*}
    This improves upon the pessimistic upper bound we used in the dense readout, interpolation regime.
\end{enumerate}

\section[Optimal scaling laws for]{Optimal scaling laws for $\beta>1/2$} \label{app:scaling}
Recall that the effective SNR for feature $i$ is
\begin{align*}
    \snr_i = \frac{(n/d) v_i^2}{\Delta + \varepsilon^{\rm BO}}.
\end{align*}
Consequently, the weakest feature (indexed by $k$) is learnable whenever $\snr_k > \lambda_\sigma$, which in turn requires $n \gtrsim k^{2\beta} d$. Using $k^{2\beta} d$ as the threshold, we distinguish two regimes for $n$, in addition to the standing assumption that $n \gg d$:

\begin{itemize}
\item When $ n \ll k^{2 \beta} d$, the weakest feature is not learnable, so $k_c < k$. 
From
\begin{align*}
\snr_{k_c}>\lambda_\sigma>\snr_{k_c + 1} \quad \text{i.e.} \quad  \frac{(n/d) v_{k_c}^2}{\Delta + \varepsilon^{\rm BO}} > \lambda_\sigma > \frac{(n/d) v_{k_c + 1}^2}{\Delta + \varepsilon^{\rm BO}},
\end{align*}
we obtain $k_c = \Theta((n/d)^{1/(2 \beta)})$, which implies $1 \ll k_c \ll k$. Recall from Result \ref{res:main},
\begin{align*}
    \varepsilon^{\rm BO} = \Theta\Big( \frac{k_c d}{n} \Big) + \Theta \Big( \sum_{i>k_c} v_i^2 \Big).
\end{align*}
From the asymptotics of $k_c$, the first term in this decomposition is $\Theta((n/d)^{1/(2 \beta)-1})$. On the other hand, 
\begin{align*}
    \sum_{i > k_c} v_i^2 &= \Theta \Big( \sum_{i > k_c} i^{-2\beta} \Big)
    = \Theta \Big( \int_{k_c}^k x^{-2 \beta} \, dx \Big) \\
    &= \Theta \Big( k_c^{1-2\beta} - k^{1-2\beta} \Big) 
    = \Theta \Big( k_c^{1-2\beta} \Big) \\
    &= \Theta((n/d)^{1/(2 \beta)-1}).
\end{align*}
Thus, the second term in the decomposition is comparable to the first one. We conclude $\varepsilon^{\rm BO} = \Theta((n/d)^{1/(2 \beta)-1})$, which scales as $n^{\frac{1}{2\beta} - 1}$ when $k, d$ are fixed.

\item When $n \gg k^{2 \beta} d$, all the features are learnable: $k_c=k$. The second term in the error decomposition vanishes, so $\varepsilon^{\rm BO} = \Theta(k d/n)$, which scales as $n^{-1}$ for fixed $k,d$.
\end{itemize}

It is worth noting from the preceding argument that when $d \ll n \ll k^{2\beta} d$, the first term in the error decomposition is comparable to the second. As $n$ increases, additional features become learnable, causing the first term to grow while the second decreases. Consequently, as long as $n \gg d$, the second term is either comparable to or negligible relative to the first. Therefore, the asymptotics for Bayes-optimal error simplifies to $\Theta(k_c d / n)$ for all $n \gg d$.

\section{Extensive-width, near-interpolation regime}\label{app:d-asymp}
We derive here the results in the table of Section~\ref{sec:discussion}. Recall that in the interpolation regime, we have $k \sim \gamma d$ and $n \sim \alpha d^2$, where $\gamma, \alpha > 0$ are fixed parameters.

\textbf{Ultra-sparse readout.} In this case, the vector $\bv$ has only finitely many nonzero components. Assume $v_1 \geq \dots \geq v_p > 0$ with $p = \Theta(1)$, while $v_{p+1} = \dots = v_k = 0$. For $i \leq p$, recall that the effective signal-to-noise ratio of $\bw_i$ is $\snr_i = (n/d) v_i^2 / (\Delta + \varepsilon^{\rm BO})$. Since $v_i = \Theta(1)$ and $\Delta = \Theta(1)$, it follows that $\snr_i \gg 1$. Therefore, feature $i$ is learnable for every $i \leq p$, implying $k_c = p = \Theta(1)$. Using \eqref{eq:err-asymp}, we obtain $\varepsilon^{\rm BO} = \Theta(k_c d / n) = \Theta(d / n) = \Theta(d^{-1})$.

For the remaining readout cases (dense, power-law, and exponentially decayed), recall that when $k_c < k$,
\begin{align*}
    \frac{(n/d) v_{k_c}^2}{\Delta + \varepsilon^{\rm BO}} > \lambda_\sigma > \frac{(n/d) v_{k_c + 1}^2}{\Delta + \varepsilon^{\rm BO}}
\end{align*}
In each of these cases, consecutive readout coefficients satisfy $v_i/v_{i+1}=\Theta(1)$. Combined with $\Delta = \Theta(1)$, the above inequality yields
\begin{align}\label{eq:kc-asym}
    v_{k_c}^2 = \Theta \Big( \frac{d}{n} \Big) = \Theta \Big( \frac{1}{d} \Big)
\end{align}
Now let us deal with each case.

\textbf{Power-law, $\beta>1/2$}. In this case $v_i^2 = \Theta(i^{-2\beta})$. Combined with \eqref{eq:kc-asym}, this gives $k_c = \Theta(d^{\frac{1}{2\beta}})$. From Result \ref{res:main},
\begin{align}\label{eq:err-decom}
    \varepsilon^{\rm BO} = \Theta\Big( \frac{k_c d}{n} \Big) + \Theta \Big( \sum_{i>k_c} v_i^2 \Big).
\end{align}
By the asymptotics of $k_c$, both terms on the right-hand side are $\Theta(d^{\frac{1}{2\beta}-1})$, and we conclude that $\varepsilon^{\rm BO} = \Theta(d^{\frac{1}{2\beta}-1})$.

\textbf{Power-law, $\beta < 1/2$}. We have $v_i^2 = \Theta(i^{-2\beta} k^{2\beta -1} )$. Combined with \eqref{eq:kc-asym}, this gives $k_c = \Theta(d)$, so the first term in \eqref{eq:err-decom} is $\Theta(1)$ and the second term is $O(1)$. We conclude $\varepsilon^{\rm BO}= \Theta(1)$.

\textbf{Power-law, $\beta=1/2$}. We have $v_i^2 = \Theta(i^{-1}/\ln k )$. Combined with \eqref{eq:kc-asym}, this gives $k_c = \Theta(d/\ln d)$, so the first term in \eqref{eq:err-decom} is $\Theta(\frac{1}{\ln d})$ while the second term is $\Theta(\frac{\ln \ln d}{\ln d})$. We conclude $ \varepsilon^{\rm BO} = \Theta(\frac{\ln \ln d}{\ln d})$.

\textbf{Exponentially decaying.} We have $v_i = \Theta (e^{-ci})$ for some constant $c>0$. Combined with \eqref{eq:kc-asym}, this gives $k_c = \Theta(\ln d)$. The first term in \eqref{eq:err-decom} dominates, and we conclude $\varepsilon^{\rm BO} = \Theta(\frac{\ln d}{d})$.

\section{Additional simulations}
The plots of $m_\sigma(\lambda)$ for three different activation functions, shown in Figure~\ref{fig:glm}, exhibit qualitatively distinct behaviors: absence of a phase transition, a continuous phase transition, and a sharp (discontinuous) phase transition. For $k_\star(\sigma) \geq 3$, which is the relevant regime for the model (\ref{eq:model}), the sharp phase transition at $\lambda_\sigma$ induces a sharp cutoff in the learning of features.

\begin{figure}[t]
  \centering
  \subfigure{\includegraphics[width=0.30\linewidth]{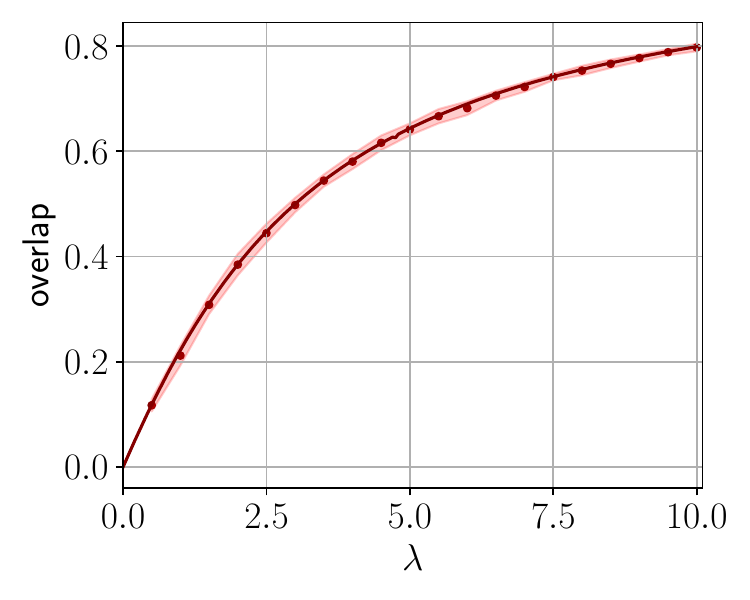}}
  \subfigure{\includegraphics[width=0.30\linewidth]{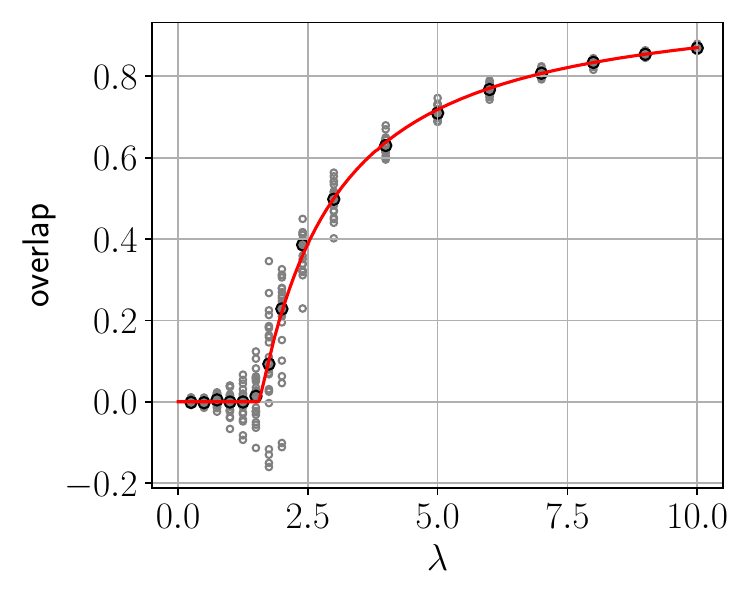}}
  \subfigure{\includegraphics[width=0.30\linewidth]{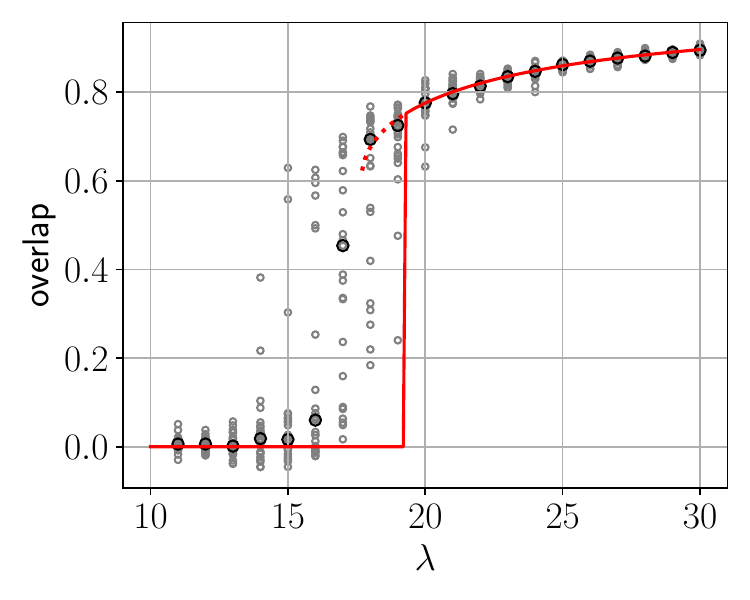}}
  \caption{Bayes-optimal overlap as a function of the SNR $\lambda$ in the generalized linear model in the high-sampling regime of Result~\ref{res:glm-highnoise}: comparison between theoretical predictions and HMC simulations initialized on the signal. The experimental data show the overlap between HMC sample and the signal, over $25$ independent experiments. The dashed line indicates the metastable solution corresponding to the local maximum near $1$ of the function $\lambda g_\sigma(q) + q + \ln(1-q)$, see the discussion concerning metastable solutions in \cite{barbier2025statistical}. From left to right: activations $\text{ReLU}(x)$, $|x|$ and $\tanh(2x)-c x$, with $c=\E_{z \sim \mathcal N(0,1)} \, z \tanh(2z)$. \textit{Common setting:} $d=1000$, $\Delta=40$.}
  \label{fig:glm}
\end{figure}

\begin{figure}[!t]
    \centering
	\subfigure{\includegraphics[width=0.32\linewidth]{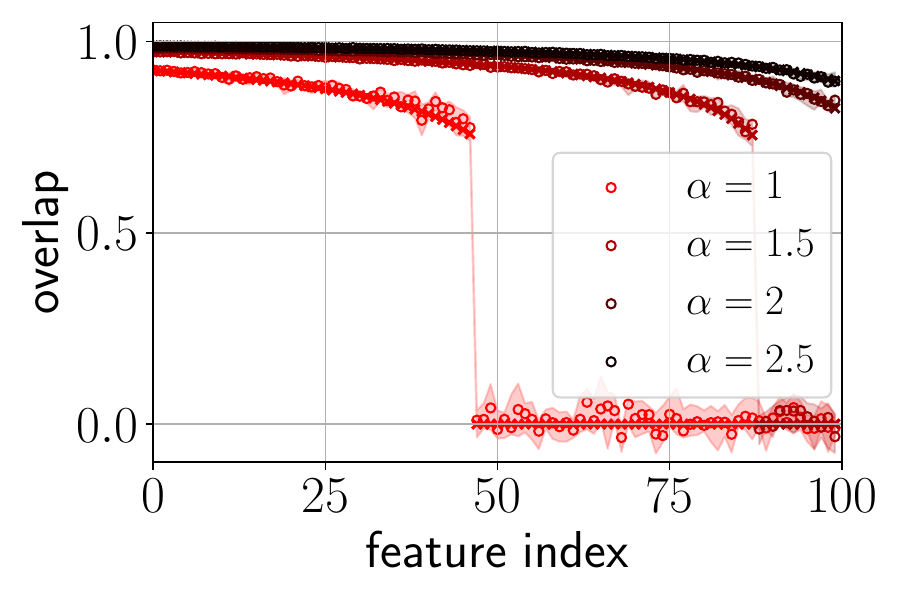}}
    \subfigure{\includegraphics[width=0.32\linewidth]{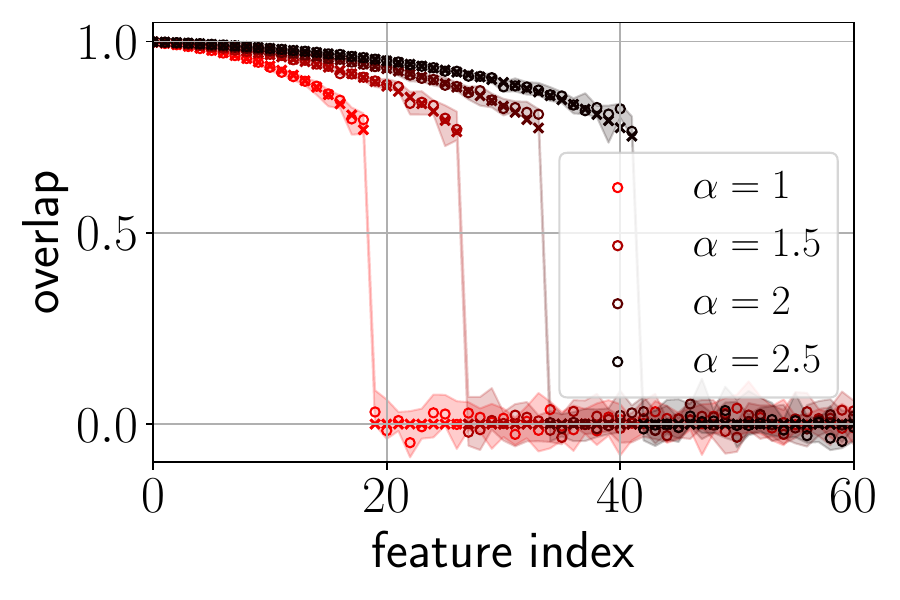}}
    \subfigure{\includegraphics[width=0.32\linewidth]{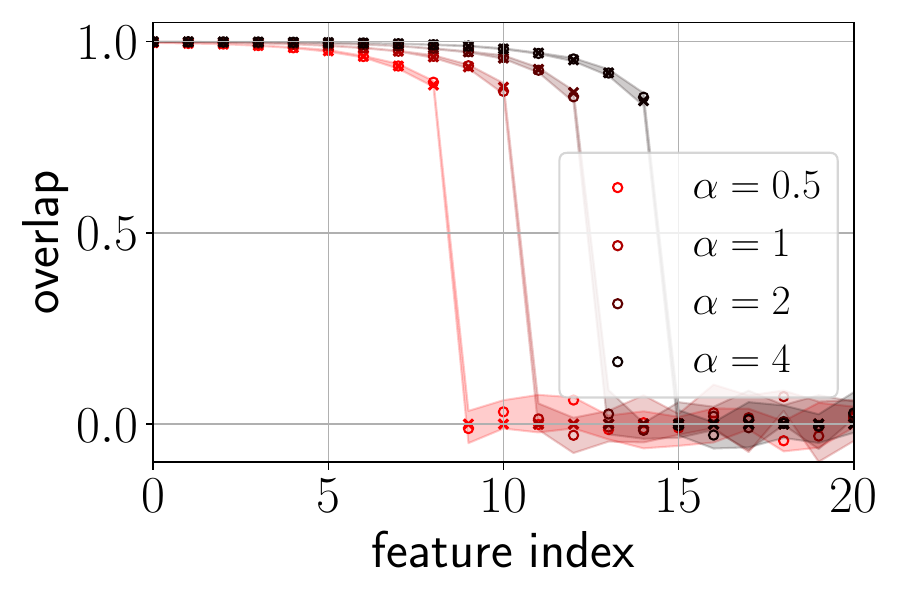}}
    \caption{Bayes-optimal feature overlaps in the neural network model \eqref{eq:model} for different types of readouts. The crosses show the theoretical prediction, while the circles show experimental results. The circles correspond to the mean overlap computed over 16 independent experiments, and the shaded region indicates the standard deviations. \textbf{Left:} dense readout with $P_v \sim \mathcal{U}([2, 5])/\sqrt{13}$. \textbf{Middle:} power-law readout with $\beta = 0.7$. \textbf{Right:} exponentially decayed readout, with $v_j \propto e^{-j/5}$. \emph{Setting:} $\sigma(x) = \tanh(2x)$, $d=200$, $k=100$, $\Delta = 0.04$.}
	\label{fig:overlaps}
\end{figure}

\section{Some exact results in the near-interpolation regime} \label{app:exact}
In the near-interpolation regime, certain exact results, such as the leading constants in the asymptotic behavior and the locations of phase transitions, can be computed explicitly as consequences of Result~\ref{res:main}. Although these results are not the primary focus of this paper, they may be of independent interest to some readers. We collect them here.

\begin{result} \label{app:dense}
Consider the near-interpolation regime with a dense readout $\bv$, such that the empirical distribution of the entries of $\sqrt{k},\bv$ converges to a limiting distribution $P_v$ in the infinite-dimensional limit. Define \begin{align}
	\varepsilon &= \lim_{d \ra \infty} \varepsilon^{\rm BO} \\
    \mathcal Q(\va) &= \lim_{\delta \ra 0} \lim_{d \ra \infty} \frac{1}{|I_{\va, \delta}|}  \sum_{ i \in I_{\va, \delta} } q_{i}
\end{align}
where $I_{\va, \delta} = \{ i \in [k]: \sqrt{k}v_i \in [\va - \delta, \va + \delta] \}$. Then we have
\begin{enumerate}
\item $(\varepsilon, \mathcal Q(\va))$ satisfies the fixed-point equations
\begin{align}
\mathcal Q(\va) &=  m_\sigma \Big( \frac{\alpha \va^2}{\gamma(\Delta + \varepsilon)} \Big)\\
\varepsilon &= g_\sigma(1)-\E_{\va \sim P_v} [\va^2 g_\sigma(\mathcal Q(\va))].
\end{align}

\item Let $\alpha_c(\va ) = \sup \{ \alpha \geq 0: \mathcal Q(\va)=0 \}$, then
\begin{align*}
    \alpha_c(\va)  &= \frac{1}{\va^2} \lambda_\sigma \gamma \Big(\Delta + g_\sigma(1) - \E_{v \sim P_v} \Big[v^2 g_\sigma \circ m_\sigma \Big( \frac{\lambda_\sigma v^2}{\va^2} \Big) \Big] \Big). 
\end{align*}
In particular, if $\max_i v_i$ converges to a limiting value $\va_{\max}$, then the first phase transition -- corresponding to the feature with the largest readout -- occurs at
\begin{align}
 \alpha =  \frac{\lambda_\sigma \gamma (\Delta + g_\sigma(1))}{\va_{\max}^2}.
\end{align}
\end{enumerate}
\end{result}

\begin{proof}
1. Fix $i \in I_{\va, \delta}$, we have
\begin{align*}
    \frac{(\va - \delta)^2}{k} \leq v_i^2 \leq \frac{(\va + \delta)^2}{k}
\end{align*}
Therefore,
\begin{align*}
    \frac{n(\va - \delta)^2}{k d (\Delta + \varepsilon)} \leq \frac{nv_i^2}{d(\Delta + \varepsilon)} \leq \frac{n(\va + \delta)^2}{kd (\Delta + \varepsilon)}
\end{align*}
Applying the non-decreasing function $m_\sigma$ on all sides of the previous inequality, we have
\begin{align*}
    m_\sigma \Big( \frac{n(\va - \delta)^2}{k d (\Delta + \varepsilon)} \Big) \leq q_{i} \leq m_\sigma \Big( \frac{n(\va + \delta)^2}{k d (\Delta + \varepsilon)} \Big).
\end{align*}
The last inequality holds for any $i \in I_{\va, \delta}$, which implies
\begin{align*}
    m_\sigma \Big( \frac{n(\va - \delta)^2}{k d (\Delta + \varepsilon)} \Big) \leq \frac{1}{|I_{\va, \delta}|} \sum_{i \in I_{\va, \delta}} q_{i} \leq m_\sigma \Big( \frac{n(\va + \delta)^2}{k d (\Delta + \varepsilon)} \Big).
\end{align*}
Taking $d \ra \infty$ in the last inequalities, we obtain
\begin{align*}
    m_\sigma \Big( \frac{\alpha(\va - \delta)^2}{\gamma (\Delta + \varepsilon)} \Big) \leq \lim_{d \ra \infty} \frac{1}{|I_{\va, \delta}|} \sum_{i \in I_{\va, \delta}} q_{i} \leq m_\sigma \Big( \frac{\alpha(\va + \delta)^2}{ \gamma(\Delta + \varepsilon)} \Big).
\end{align*}
Taking $\delta \ra 0$, we have
\begin{align} \label{eq:fpe-dense-1}
    \mathcal Q(\va) = m_\sigma \Big( \frac{\alpha \va^2}{\gamma (\Delta + \varepsilon)} \Big)
\end{align}

\begin{align*}
    \varepsilon^{\rm BO} \sim \sum_{i=1}^k v_i^2 (g_\sigma(1) - g_\sigma(q_{i})),
\end{align*}
taking the limit $d \ra \infty$, and using an argument similar to the derivation of \eqref{eq:fpe-dense-1}, we obtain
\begin{align}
    \varepsilon = g_\sigma(1) - \E_{\va \sim P_v}[\va^2 g_\sigma(\mathcal Q(\va))]
\end{align}

2. From the definition of $\alpha_c(\va)$ and equation \eqref{eq:fpe-dense-1}, using the fact that $\lambda_\sigma = \sup \{ \lambda \geq 0: m_\sigma(\lambda) = 0 \} $, we have 
\begin{align*}
    \frac{\alpha_c(\va) \va^2}{\Delta + \varepsilon} = \lambda_\sigma.
\end{align*}
Therefore,
\begin{align*}
    \alpha_c(\va) &= \frac{1}{\va^2} \lambda_\sigma \gamma (\Delta + \varepsilon) \\
    &= \frac{1}{\va^2} \lambda_\sigma \gamma (\Delta + g_\sigma(1) - \E_{v \sim P_v}[v^2 g_\sigma \circ \mathcal Q(v)]) \\
    &= \frac{1}{\va^2} \lambda_\sigma \gamma \Big(\Delta + g_\sigma(1) - \E_{v \sim P_v} \Big[v^2 g_\sigma \circ m_\sigma \Big( \frac{\alpha_c(\va) v^2}{\gamma(\Delta +\varepsilon)} \Big) \Big] \Big) \\
    &= \frac{1}{\va^2} \lambda_\sigma \gamma \Big(\Delta + g_\sigma(1) - \E_{v \sim P_v} \Big[v^2 g_\sigma \circ m_\sigma \Big( \frac{\lambda_\sigma v^2}{\va^2} \Big) \Big] \Big) 
\end{align*}
If $\va = \va_{\max}$, then for any $v \in \supp(P_v)$, we have $\lambda_\sigma v^2/\va^2 \leq \lambda_\sigma$, so $m_\sigma(\lambda_\sigma v^2/\va^2 ) = 0$. The formula simplifies to 
\begin{align*}
    \alpha_c(\va_{\max}) = \frac{\lambda_\sigma \gamma (\Delta + g_\sigma(1))}{\va_{\max}^2}.
\end{align*}
\end{proof}

\begin{result}\label{res:scaling}
Consider the near-interpolation regime with power-law readouts with exponent $\beta$, i.e. 
\begin{align*}
    v_i^2 = \frac{i^{-2\beta}}{\sum_{j=1}^k j^{-2\beta}}, \quad i \in [k].
\end{align*}
\begin{enumerate}
\item If $k_c < k$ and $i = \floor{x k_c}$ for some $x \in (0,1)$, then
\begin{align}
    \lim_{d \ra \infty }q_{i} =  m_\sigma(\lambda_\sigma x^{-2\beta}):= \phi (x)
\end{align}
In other words, the overlaps $(q_i)_{i=1}^k$ in the infinite dimension limit can be compactly described by a limiting profile $\phi$.

\item For $\beta > 1/2$, let
\begin{align}
    \bar k_c = \lim_{d \ra \infty} k_c d^{-\frac{1}{2\beta}}, \quad \bar \varepsilon^{\rm BO} = \lim_{d \ra \infty} \varepsilon^{\rm BO} d^{ -\frac{1}{2\beta} + 1 },
\end{align}
Then, letting $z(\beta) := \sum_{i=1}^\infty i^{-2\beta}$ and since $\varepsilon=0$,
\begin{align}
    \bar k_c &= \left( \frac{\alpha}{\lambda_\sigma z(\beta) \Delta }\right)^{\frac{1}{2\beta}} \\
    \bar \varepsilon^{\rm BO} &= \frac{\bar{k}_c^{1-2\beta}}{z(\beta)} \left( \int_0^1 \frac{g_\sigma(1)-g_\sigma(\phi(x))}{x^{2\beta}} dx + \frac{g_\sigma(1)}{(2\beta -1) } \right).
\end{align}

\item For $\beta < 1/2$, let
\begin{align}
    \bar k_c = \lim_{d \ra \infty} k_c/k, \quad \varepsilon = \lim_{d \ra \infty} \varepsilon^{\rm BO}.
\end{align}
Then $\varepsilon$ is a solution in $(0, \infty)$ of 
\begin{align}
    \frac{g_\sigma(1) - \varepsilon}{1-2\beta} = \int_0^1 x^{-2\beta} g_\sigma \circ m_\sigma \Big( \frac{\alpha(1-2\beta)}{\gamma(\Delta + \varepsilon)} x^{-2\beta} \Big)dx
\end{align}
and
\begin{align}
    \bar k_c = \Big(\frac{(1-2\beta) \alpha}{\lambda_\sigma \gamma (\Delta + \varepsilon)} \Big)^{\frac{1}{2\beta}}
\end{align}
Moreover, in the thermodynamic limit, all features are learnable if and only if $\alpha > \alpha_\star$, where
\begin{align}
   \frac{\alpha_\star}{\lambda_\sigma \gamma} = \frac{ \Delta + g_\sigma(1)}{1-2\beta} - \int_0^1 \frac{g_\sigma(\phi (x))}{x^{2\beta}}\, dx,
\end{align}

\end{enumerate}
\end{result}

\textit{Derivation.} 1. We have
\begin{align}\label{eq:threshold}
\frac{n v_{k_c+1}^2}{d(\Delta + \varepsilon^{\rm BO})}
< \lambda_\sigma <
\frac{n v_{k_c}^2}{d(\Delta + \varepsilon^{\rm BO})} .
\end{align}
From the inequality on the left, it is clear that $k_c \to \infty$ as $d \to \infty$. This implies $\lim_{d \ra \infty} v_{k_c}/v_{k_c+1}=1$, which yields
\begin{align}\label{eq:threshold-lim}
    \lim_{d \ra \infty} \frac{n v_{k_c}^2}{d(\Delta + \varepsilon^{\rm BO})}  = \lambda_\sigma.
\end{align}
Next, let $x \in (0,1)$ and $i = \floor{x k_c}$. Recall that
\begin{align}\label{eq:proof-scaling-1}
q_{i} \simeq m_\sigma \Big( \frac{n v_i^2}{ d(\Delta + \varepsilon^{\rm BO}) } \Big)
\end{align}

Sending $d \ra \infty$ on both sides of (\ref{eq:proof-scaling-1}), and using the fact that
\begin{align}
    \lim_{d \ra \infty} \frac{n v_i^2}{ d(\Delta + \varepsilon^{\rm BO} ) } = \lim_{d \ra \infty} \frac{n v_{k_c}^2}{ d(\Delta + \varepsilon^{\rm BO} )} \frac{v_i^2}{v_{k_c}^2} = \lambda_\sigma x^{-2\beta},
\end{align}
we obtain
\begin{align}
    \lim_{d \ra \infty} q_{i} = m_\sigma(\lambda_\sigma x^{-2\beta})  = \phi (x)
\end{align}

2. From~\eqref{eq:threshold-lim},
straightforward calculations yield
\begin{equation}
    \lim_{d \ra \infty} k_c \, d^{-\frac{1}{2\beta}} = \left( \frac{\alpha}{\lambda_\sigma z(\beta) (\Delta + \varepsilon) }\right)^{\frac{1}{2\beta}}.
\end{equation}

Next, using $q_{i} \simeq \phi(i/k_c)$ for $i \leq k_c$, the error contributed by learnable features is
\begin{align}
&\sum_{i \leq k_c} \frac{i^{-2\beta}}{z(\beta)} 
(g_\sigma(1)-g_\sigma \circ \phi(i/k_c)) \nonumber \\
&= k_c^{1-2\beta} \frac{1}{k_c} \sum_{i \leq k_c}  \frac{(i/k_c)^{-2\beta}}{z(\beta)} (g_\sigma(1)-g_\sigma \circ \phi(i/k_c)) \\
&\simeq k_c^{1-2\beta} \int_0^1 \frac{x^{-2\beta}}{z(\beta)} (g_\sigma(1) - g_\sigma(\phi (x))) dx.
\end{align}
where in the last step, the averaged sum over $i \leq k_c$ is asymptotically the integral over $[0,1]$.

On the other hand, since $q_i=0$ as $i>k_c$, the error that comes from non-learnable features is
\begin{align}
\sum_{i>k_c} v_i^2 g_\sigma(1) \simeq g_\sigma(1) \sum_{i>k_c} \frac{i^{-2\beta}}{z(\beta)} 
\simeq \frac{g_\sigma(1)}{(2\beta-1) z(\beta)} k_c^{1-2\beta}
\end{align}
This implies $\varepsilon^{\rm BO} \propto d^{\frac{1}{2\beta}-1}$, so $\varepsilon = \lim_{d \ra \infty} \varepsilon^{\rm BO} = 0$. Plugging this value of $\varepsilon$ into the previous formulas, we obtain the results for $\beta > 1/2$.

3. We have
\begin{align}
    v_i^2 = \frac{i^{-2 \beta}}{\sum_{j=1}^k j^{-2 \beta}} \simeq (1-2\beta) i^{-2\beta} k^{2 \beta -1}
\end{align}

From~\eqref{eq:threshold-lim}, straightforward calculations yield
\begin{align}
    \lim_{d \ra \infty} k_c/k = \Big(\frac{(1-2\beta) \alpha}{\lambda_\sigma \gamma (\Delta + \varepsilon)} \Big)^{\frac{1}{2\beta}}
\end{align}
In this case it is possible for all features to be learnable. Let us denote 
\begin{align}
    \alpha_\star = \lim_{d \ra \infty} \alpha_{c}(k)
\end{align}
where we recall that $\alpha_{c}(k)$ is the threshold of $\alpha$ for the feature $k$ to be learnable. Let $\varepsilon$ be the BO error at this value of $\alpha$. We have
\begin{align*}
    \alpha_c(k) &\simeq \frac{\lambda_\sigma (\Delta + \varepsilon)}{d v_k^2} \\
    &= \frac{\lambda_\sigma}{dv_k^2} \Big( \Delta + g_\sigma(1) - \sum_{i=1}^k v_i^2 g_\sigma(q_{i}) \Big) \\
    &= \frac{\lambda_\sigma}{dv_k^2} \Big( \Delta + g_\sigma(1) - \sum_{i=1}^k v_i^2 g_\sigma \circ m_\sigma \Big( \frac{\alpha_c(k) d v_i^2}{\Delta + \varepsilon} \Big) \Big) \\
    &= \frac{\lambda_\sigma}{dv_k^2} \Big( \Delta + g_\sigma(1) - \sum_{i=1}^k v_i^2 g_\sigma \circ m_\sigma \Big( \lambda_\sigma \frac{v_i^2}{v_k^2} \Big) \Big) 
\end{align*}
Using the fact that $v_k^2 \sim (1-2\beta)/k$, $v_i^2/v_k^2 = (i/k)^{-2\beta}$ and $v_i^2 \sim (1-2\beta) i^{-2\beta} k^{2\beta-1}$, the last expression can be written as a Riemann sum, which in the limit $d \ra \infty$ gives
\begin{align}
   \alpha_\star = \frac{\lambda_\sigma \gamma (\Delta + g_\sigma(1))}{1-2\beta} - \lambda_\sigma \gamma \int_0^1 \frac{g_\sigma(\phi (x))}{x^{2\beta}}\, dx,
\end{align}

Let us derive the fixed-point equation for $\varepsilon$. We have
\begin{align*}
    \varepsilon &= g_\sigma(1) - \sum_{i=1}^k v_i^2 g_\sigma(q_{i}) \\
    &\sim g_\sigma(1) - (1-2\beta)k^{2\beta -1} \sum_{i=1}^k i^{-2\beta} g_\sigma \circ m_\sigma \Big( \frac{\alpha d v_i^2}{\Delta + \varepsilon} \Big) \\
    &\sim g_\sigma(1) - (1-2\beta)k^{2\beta -1} \sum_{i=1}^k i^{-2\beta} g_\sigma \circ m_\sigma \Big( \frac{\alpha(1-2\beta)}{\gamma(\Delta + \varepsilon)} (i/k)^{-2\beta} \Big) \\
    &\sim g_\sigma(1) - (1-2\beta) \frac{1}{k} \sum_{i=1}^k (i/k)^{-2\beta} g_\sigma \circ m_\sigma \Big( \frac{\alpha(1-2\beta)}{\gamma(\Delta + \varepsilon)} (i/k)^{-2\beta} \Big) \\
    &\sim g_\sigma(1) - (1-2\beta) \int_0^1 x^{-2\beta} g_\sigma \circ m_\sigma \Big( \frac{\alpha(1-2\beta)}{\gamma(\Delta + \varepsilon)} x^{-2\beta} \Big)dx
\end{align*}
which gives the fixed-point equation for $\varepsilon$.

\begin{remark}
The following properties of the limiting profile $\phi (x) = m_\sigma(\lambda_\sigma x^{-2\beta})$ can be derived from properties of $m_\sigma$ in Appendix \ref{app:msig}.
\begin{enumerate}
\item $\phi (x)$ is non-increasing
\item $\lim_{x \ra 0} \phi (x)=1$, $\lim_{x \ra 1} \phi (x) = q_\sigma$.
\item $\phi (x) = 1 - cx^{2\beta} + o(x^{2 \beta})$ as $x \ra 0^+$, for some $c>0$
\end{enumerate}
\end{remark}

\section{Numerical methods}\label{app:numerics}

\subsection{Solving the fixed-point equations}\label{app:fpe}
From the system of fixed-point equations, we obtain the following equation for $\varepsilon$:
\begin{align}\label{eq:fixedPointUnique}
    \varepsilon = g_\sigma(1) - \sum_{i=1}^k v_i^2 g_\sigma \circ m_\sigma \Big( \frac{nv_i^2/d}{\Delta + \varepsilon} \Big).
\end{align}
This can be quickly solved by a standard numerical method such as bisection.

\subsection{HMC simulation}\label{app:hmc}
The theoretical predictions for the overlaps are validated by HMC simulation. In general, sampling from the posterior is difficult since the posterior landscape can be extremely complicated: with random initialization, it can take exponentially long to reach the true equilibrium. However, if the initialization lies in the same basin of attraction as the equilibrium state, convergence is achieved with reasonable computational cost. We initialize such that the first $k_c$ rows coincide with the ground truth, and the remaining rows are random. Note that the HMC does not validate the prediction of $k_c$ itself.

\subsection{Numerical evaluation of Bayes-optimal error}\label{app:gibbs}
Directly evaluating the Bayes-optimal error from its definition is computationally expensive. Indeed, computing $\E[y(\bx, \bW) \mid \mathcal D]$ requires drawing many independent samples from the posterior distribution $P(\bW \mid \mathcal D)$, and each sample requires a computational time on the order of $10^2$ seconds on a GPU. To overcome this difficulty, we compute the Bayes-optimal error through \emph{Gibbs error}, defined as
\begin{align}
   \varepsilon^{\rm Gibbs} = \E_{\mathcal D, \bW^0, \bW^1} \E_{\bx \sim \mathcal N(0, \bI_d)} (y(\bx, \bW^0) - y(\bx, \bW^1))^2,
\end{align}
where $\bW^1$ is a sample of $P(\bW \mid \mathcal D)$. It is known that $\varepsilon^{\rm BO} = \frac{1}{2} \varepsilon^{\rm Gibbs}$. This identity is a straightforward generalization of the identity $\E(X-X')^2 = 2 \E( X - \E X )^2 $ for any i.i.d. random variables $X, X'$. In practice, one only needs $10$--$10^2$ samples of $(\mathcal D, \bW^0, \bW^1)$ due to small fluctuations with respect to these variables, while averaging over $\bx$ requires $10^4$--$10^5$ samples, each of which can be generated at small cost. Computing the optimal error via Gibbs error is much more efficient, as for each fixed $(\mathcal D, \bW^0)$, only one posterior sample from $P(\bW \mid \mathcal D)$ is needed. 

\subsection{Training with SGD} \label{app:sgd}
We use the \textsc{Adam} optimizer implemented in \texttt{PyTorch} to train the student network. We note here some important points in the training:
\begin{enumerate}
\item The student’s features are normalized after each epoch, preventing their norms from growing excessively and resulting in more stable and well-behaved training dynamics.
\item Performance degrades when the student's readout is fixed to match the teacher during training, as the student is more likely to get stuck in local minima. In contrast, when the student readout is learnable, it still recovers that of the teacher, and the possibility for the readout to change during training helps the student escape local minima.
\end{enumerate}

\section[Additional Adam experiments]{Additional \textsc{Adam} experiments}\label{app:adam}
Figure~\ref{fig:adam} repeats the middle and right panels of Fig.~\ref{fig:sgd-student} for a power-law readout with $\beta=1.2$.

\begin{figure}[t]
\centering
\subfigure{\includegraphics[width=0.45\linewidth]{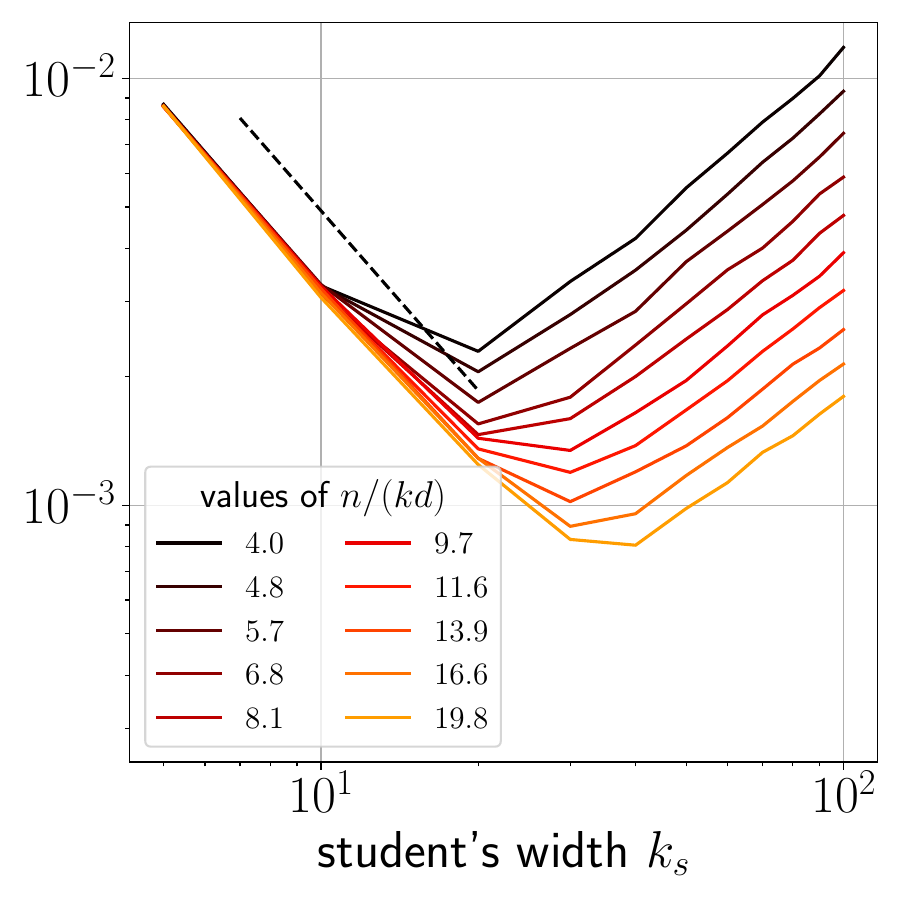}}
\hfill
\subfigure{\includegraphics[width=0.45\linewidth]{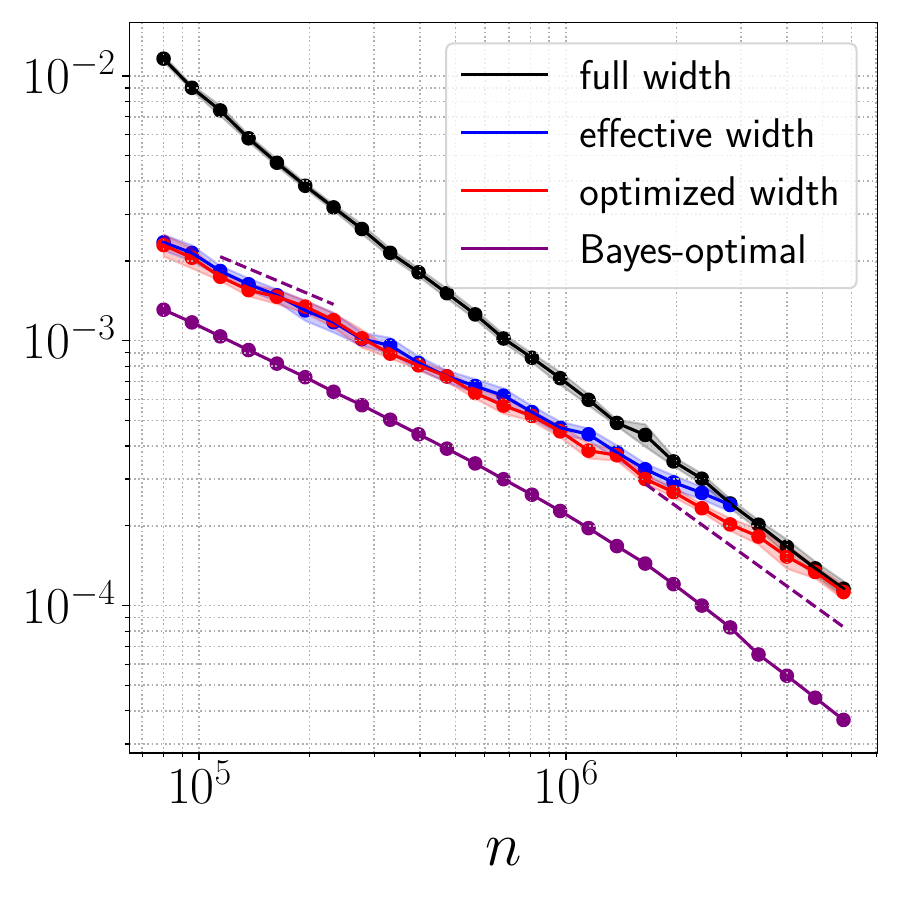}}
\vspace{-5pt}
\caption{\textsc{Adam}-trained students at $\beta=1.2$, as in the middle and right panels of Fig.~\ref{fig:sgd-student}. \textbf{Left:} generalization error vs.\ student width $k_s$ at several sample rates $n/(kd)$; dashed line: slope $1-2\beta=-1.4$. \textbf{Right:} generalization error vs.\ data size $n$ for the \emph{full-width} student ($k_s=k$, black), the \emph{effective-width} student ($k_s=k_c$ from theory, blue) and the \emph{optimized-width} student (red), together with the Bayes-optimal prediction (purple). Dashed lines: slopes $1/(2\beta)-1=-7/12$ and $-1$. Teacher $\sigma(x)=\tanh(2x)$, $k=100$, $d=200$, $\Delta=0.01$; 3 batches, $10^4$ epochs, learning rate $0.003$, learnable readout initialized at $1/\sqrt{k_s}$; mean over 9 runs with a fresh teacher each, shaded regions for $\pm1$ standard deviation.}
\label{fig:adam}
\end{figure}

\section{On the fixed-point equations}\label{app:potential}

In all experiments in the paper, the fixed-point equations~\eqref{eq:fpe} have a unique solution, which we checked numerically. In general, however, uniqueness does not hold. For example, take $\sigma(z)=\He_3(z)/\sqrt{6}$, $v_i^2=1/k$, $n/(kd)=1.07$ and $\Delta=0.02$. By symmetry, $q_i=q$ for all $i\in[k]$, and~\eqref{eq:fpe} has three solutions: $q=0$ and two values of $q$ in $(0,1)$.

To select the relevant solution, we propose to characterize the overlaps as the maximizer $q^\star$ of the potential
\begin{align}\label{eq:potential-F}
    F(q) = -\frac{n}{d} \ln(\Delta + \varepsilon) + \sum_{i=1}^k [q_i + \ln(1-q_i)], \qquad \varepsilon = \sum_{i=1}^k v_i^2 (g_{\sigma}(1) -g_{\sigma}(q_i)),
\end{align}
where $q=(q_1, \dots, q_k)$. This potential extends the replica-symmetric potential of~\cite{barbier2025statistical} beyond the quadratic data regime and dense readouts. Our cavity argument cannot derive $F$ (a known limitation of cavity arguments), so $F$ is an educated guess. Its advantage is that its maximizer is expected to be unique away from transition points.

The maximizer $q^\star$ solves~\eqref{eq:fpe} up to a vanishing correction. Indeed, $q^\star_i$ maximizes $q_i\mapsto F(q_1^\star, \dots, q_i, \dots, q_k^\star)$. Let $a=\Delta + \varepsilon^\star_{-i}$ and $x=v_i^2 (g_\sigma(1) - g_\sigma(q_i))$. Using $\ln(a + x) \simeq \ln a + x/a$ for $x \ll a$, we have, up to terms independent of $q_i$,
\begin{align*}
    F(q_1^\star, \dots, q_i, \dots, q_k^\star) &\simeq -\frac{n}{d}\,\frac{v_i^2 (g_\sigma(1) - g_\sigma(q_i))}{\Delta + \varepsilon^\star_{-i}} + q_i + \ln(1-q_i) \\
    &= \frac{n v_i^2/d}{\Delta + \varepsilon^\star_{-i}}\, g_\sigma(q_i) + q_i + \ln(1-q_i) + \text{const}.
\end{align*}
By definition~\eqref{defs}, the maximizer is $m_\sigma\big(\frac{n v_i^2/d}{\Delta + \varepsilon^\star_{-i}}\big)$, and using $\varepsilon^\star_{-i}\simeq\varepsilon^\star$ as in Section~\ref{sec:derivations} gives~\eqref{eq:fpe}. Consequently, if~\eqref{eq:fpe} has a unique solution, this solution is the maximizer of $F$.

Setting the derivative of $F$ with respect to $q_i$ to zero gives the stationarity conditions
\begin{align}\label{eq:stationary}
\frac{q_i}{1-q_i} &= \frac{(n/d) v_i^2}{\Delta + \varepsilon}\, g'_\sigma(q_i), \qquad i\in[k].
\end{align}
This system, and more generally the stationarity conditions of the replica-symmetric potential in Bayes-optimal inference problems, is informative about algorithms such as AMP. The computational-to-statistical gap is related to the existence of multiple solutions of this system.

In summary, the overlaps can be characterized in three ways. The fixed-point equations~\eqref{eq:fpe} follow from an intuitive leave-one-out argument, and typically have a unique solution, but not always, as the example above shows. The potential~\eqref{eq:potential-F} determines the overlaps uniquely, but cannot be derived by the same leave-one-out argument. The stationarity conditions~\eqref{eq:stationary} are informative about algorithms, but do not indicate which of their solutions is the equilibrium.

\section{Two hidden layers}\label{sec:2layer}

We extend the analysis to networks with two hidden layers. The starting point is a potential obtained heuristically from the replica-symmetric theory of multi-layer networks with dense readouts~\cite{barbier2025statistical}; everything below is derived from it, and Section~\ref{sec:2layer-hmc} checks it with HMC. The results are gathered in Result~\ref{res:2layer}. The main differences with one hidden layer are:
\begin{itemize}
\item \emph{First transition.} With one layer, the first feature is learned when its signal-to-noise ratio $(n/d)v_1^2/(\Delta+\varepsilon^{\rm BO})$ reaches $\lambda_\sigma$, that is at $n=\Theta(d/v_1^2)$. The first transition therefore depends on the readout: $n=\Theta(d)$ when $v_1=\Theta(1)$, as for power-law readouts with $\beta>1/2$, but $n=\Theta(kd)$ for dense readouts. With two layers, nothing is learned below $n=\Theta(k_1d)=\Theta(d^2)$, \emph{whatever the readout}, and learning always starts at this scale (Section~\ref{sec:nolearn}). There the first layer and many second-layer features are learned together, in a single jump: neither layer can be learned without the other.
\item \emph{Exponent independent of the readout.} With one layer, the readout sets the exponent: for power-law readouts, the Bayes-optimal error decays as $n^{1/(2\beta)-1}$ in the feature-learning regime and as $n^{-1}$ in the refinement regime, and power-law readouts are needed for a non-trivial exponent. With two layers, the error decays as $n^{-1}$ in both regimes and for every readout; the readout only changes the error by a constant factor. Power-law readouts serve as an example (Section~\ref{sec:2layer-pl}) and in the numerics.
\item \emph{Effective width.} For power-law readouts, the number of learned features is $k_c=\min\{k,\Theta((n/d)^{1/(2\beta)})\}$ with one layer, and the number of learned second-layer features is $k_c=\min\{k_2,\Theta((n/k_1)^{1/(2\beta)})\}$ with two layers: the input dimension $d$ is replaced by the width $k_1$ of the first layer.
\end{itemize}

\subsection{Model and potential}
The teacher is
\begin{align}\label{eq:2layer}
y_\mu=\bv^\top\sigma\big(\bW^{(2)}\sigma(\bW^{(1)}\bx_\mu)\big)+\sqrt\Delta\,z_\mu,\qquad\mu\in[n],
\end{align}
with $\bW^{(1)}\in\R^{k_1\times d}$ and $\bW^{(2)}\in\R^{k_2\times k_1}$ normalized as $\bW^0$ in~\eqref{eq:model}: their rows have unit norm, and are orthonormal. Inputs and noise are standard Gaussian, and the readout $\bv\in\R^{k_2}$ is known, with $v_1\geq\dots\geq v_{k_2}\geq0$ and $\|\bv\|=1$, and otherwise arbitrary. As for one layer, $\sigma$ has Hermite coefficients $\mu_0=\mu_1=\mu_2=0$, and we normalize $g_\sigma(1)=1$. We take $\Delta=\Theta(1)$ and work in the regime
\begin{align}\label{eq:regime2}
d\gg1,\qquad k_1=\Theta(d),\qquad k_2=\Theta(d),\qquad n\gtrsim k_1d,
\end{align}
that is, $n=\Theta(d^2)$ at the first transition (Section~\ref{sec:nolearn}) and $n\gg d^2$ beyond it.

Let $q_0$ be the overlap of the rows of $\bW^{(1)}$ with the teacher, and $q_i$ that of the $i$-th row of $\bW^{(2)}$. The natural sample rate is the number of samples per first-layer parameter,
\begin{align}\label{eq:alpha1}
\alpha_1:=\frac{n}{k_1d}.
\end{align}
Adapting the replica-symmetric potential of~\cite{barbier2025statistical} to a structured readout, we use
\begin{align}\label{eq:F}
F(q_0,\bq)&=-\alpha_1\ln(\Delta+\varepsilon)+q_0+\ln(1-q_0)+\frac{1}{d}\sum_{i=1}^{k_2}\big(q_i+\ln(1-q_i)\big),
\nonumber\\
\varepsilon&=\sum_{i=1}^{k_2}v_i^2\big(1-g_\sigma(Gq_i)\big), \qquad G=g_\sigma(q_0),
\end{align}
take the overlaps to be its global maximizer, and $\varepsilon^{\rm BO}\simeq\varepsilon$.

A single overlap $q_0$ suffices for the first layer because its rows play symmetric roles: the law of the teacher is invariant under a permutation of the first-layer units, applied to the rows of $\bW^{(1)}$ and to the columns of $\bW^{(2)}$. The overlaps of the first-layer rows are therefore exchangeable; if they concentrate, as assumed for one layer, they concentrate on a common value $q_0$. The second-layer rows, in contrast, are distinguished by their readouts $v_i$, and have different overlaps $q_i$.

\begin{result}[Two-hidden-layer scaling laws]\label{res:2layer}
In the regime~\eqref{eq:regime2}, under the potential~\eqref{eq:F}:
\begin{enumerate}
\item \textbf{First transition.} There is a constant $\alpha_{1c}>0$, depending only on $\sigma$ and $\Delta$, such that for $n<\alpha_{1c}k_1d$ nothing is learned and $\varepsilon^{\rm BO}=1$. For any readout, learning starts at $n=\Theta(k_1d)$, through a joint discontinuous jump of the first-layer overlap and of the top second-layer overlaps.
\item \textbf{Error.} For any readout with $\|\bv\|=1$ and $n\gg k_1d$, $\varepsilon^{\rm BO}=\Theta(k_1d/n)$.
\item \textbf{Effective width.} For power-law readouts with $\beta>1/2$, above the first transition, the number of learned second-layer rows is $k_c=\min\{k_2,\Theta((n/k_1)^{1/(2\beta)})\}$.
\end{enumerate}
\end{result}
The three items are derived in Sections~\ref{sec:nolearn}, \ref{sec:2layer-err} and~\ref{sec:2layer-pl}.

\subsection{Fixed-point equations}\label{sec:fpe2}
\paragraph{Second layer: a one-layer problem with input dimension $k_1$ and a noisy activation.}
For $G\in[0,1]$, let $\sigma_G(z):=\E_\xi\,\sigma(\sqrt G z+\sqrt{1-G}\,\xi)$ with $\xi\sim\mathcal N(0,1)$. Its Hermite coefficients are $\mu_\ell G^{\ell/2}$, so
\begin{align*}
g_{\sigma_G}(q)=\sum_{\ell\geq3}\frac{\mu_\ell^2}{\ell!}G^\ell q^\ell=g_\sigma(Gq),\qquad g_{\sigma_G}(1)=g_\sigma(G).
\end{align*}
The term $g_\sigma(Gq_i)$ in $\varepsilon$ is thus the one-layer kernel of $\sigma_G$: the second layer sees the hidden representation through a noisy channel whose quality is set by the first-layer overlap. Define $m_G:=m_{\sigma_G}$ and $\lambda_G:=\lambda_{\sigma_G}$ as in~\eqref{defs},
\begin{align*}
m_G(\lambda)=\arg\max_{q\in[0,1)}\big\{\lambda g_\sigma(Gq)+q+\ln(1-q)\big\},\qquad
\lambda_G=\inf_{q\in(0,1)}\frac{-q-\ln(1-q)}{g_\sigma(Gq)}.
\end{align*}
For $G>0$, $\sigma_G$ again has $\mu_1=\mu_2=0$, so $m_G$ is non-decreasing, vanishes below $\lambda_G$, and jumps at $\lambda_G$ from $0$ to a value $q_G>0$. The threshold $\lambda_G$ decides which second-layer rows are learned.

Maximizing $F$ in a single $q_i$ gives the analogue of~\eqref{eq:fpe}:
\begin{align}\label{eq:fpe2}
q_i=m_G(\snr_i),\qquad\snr_i:=\frac{(n/k_1)\,v_i^2}{\Delta+\varepsilon},\qquad i\in[k_2].
\end{align}
Compared with one layer, where $\snr_i=(n/d)v_i^2/(\Delta+\varepsilon)$, the input dimension $d$ is replaced by the width $k_1$ of the first layer, and $\lambda_\sigma$ by $\lambda_G$. Let $k_c:=|\{i:q_i>0\}|$ be the effective width of the second layer.

\paragraph{First layer.}
Stationarity gives
\begin{align}\label{eq:q0}
\frac{q_0}{1-q_0}=\alpha_1\,\frac{g_\sigma'(q_0)\,S}{\Delta+\varepsilon},\qquad
S:=\sum_{i=1}^{k_2}v_i^2\,q_i\,g_\sigma'(Gq_i).
\end{align}

\paragraph{Two kinds of equilibria.} The solutions of~\eqref{eq:fpe2} and~\eqref{eq:q0} are of two kinds:
\begin{itemize}
\item the trivial point $q_0=0$, $\bq=\boldsymbol0$, with $\varepsilon=1$;
\item $q_0>0$ and $\bq\neq\boldsymbol0$, with each $q_i$ fixed by $q_0$ and $\varepsilon$ through~\eqref{eq:fpe2}, $G=g_\sigma(q_0)$.
\end{itemize}
Indeed, if $q_0=0$, then $G=g_\sigma(0)=0$, so $g_\sigma(Gq)=0$ for every $q$: the objective defining $m_0$ reduces to $q+\ln(1-q)$, maximized at $q=0$, so $m_0\equiv0$ (and $\lambda_0=+\infty$), and~\eqref{eq:fpe2} gives $\bq=\boldsymbol0$. If $\bq=\boldsymbol0$, then $S=0$, so~\eqref{eq:q0} gives $q_0=0$. Neither layer can therefore be learned without the other. 

\paragraph{Numerical solution.} At fixed $q_0$, the potential is that of one hidden layer: with $G=g_\sigma(q_0)$,
\begin{align*}
 F(q_0,\bq)&=-\frac{n}{k_1 d}\ln(\Delta+\varepsilon)+\frac{1}{d}\sum_{i=1}^{k_2}\big(q_i+\ln(1-q_i)\big)+\text{const},\\
\varepsilon&=\sum_{i=1}^{k_2}v_i^2\big(g_{\sigma_G}(1)-g_{\sigma_G}(q_i)\big)+1-g_\sigma(G),
\end{align*}
which is the one-layer potential with activation $\sigma_G$ and input dimension $k_1$, up to the constant shift $1-g_\sigma(G)$ of the error. The second layer is therefore solved by the one-layer method, and $F$ is then maximized in $q_0$; we alternate between the two.
\begin{itemize}
\item \emph{Second layer, at fixed $q_0$.} As for one layer, the fixed-point equations reduce to the single scalar equation for $\varepsilon$ above, since every $q_i=m_G(\snr_i)$ is a function of $\varepsilon$. It is solved by a damped fixed-point iteration (the new value is the average of the old one and of the right-hand side).
\item \emph{First layer, at fixed $\bq$.} $F$ is maximized over $q_0$ exactly: on a grid of $-\log_{10}(1-q_0)\in[0,16]$, refined by additional search around the best grid point.
\item The two steps alternate, starting from $q_0=1-10^{-3}$, until $\ln(1-q_0)$ changes by less than $10^{-6}$. The value of $F$ at the solution is compared with the trivial point, $F(0,\boldsymbol0)=-\alpha_1\ln(\Delta+1)$, and the larger one is kept.
\end{itemize}

\subsection{No learning below \texorpdfstring{$n\asymp k_1d$}{n ~ k1 d}}\label{sec:nolearn}
By Section~\ref{sec:fpe2}, an equilibrium is non-trivial if and only if $q_0>0$. We show that the global maximizer of $F$ becomes non-trivial at $\alpha_1=\Theta(1)$.

\emph{Upper bound.} The trivial point has $\varepsilon=1$, so $F(0,\boldsymbol0)=-\alpha_1\ln(\Delta+1)$. For any $(q_0,\bq)$,
\begin{align*}
&F(q_0,\bq)-F(0,\boldsymbol0)\\
&\quad=-\alpha_1\ln\frac{\Delta+\varepsilon}{\Delta+1}+q_0+\ln(1-q_0)+\frac{1}{d}\sum_{i=1}^{k_2}\big(q_i+\ln(1-q_i)\big)\\
&\quad\leq-\alpha_1\ln\frac{\Delta+\varepsilon}{\Delta+1}+q_0+\ln(1-q_0),\qquad\text{since }q+\ln(1-q)\leq0\text{ on }[0,1)\\
&\quad=-\alpha_1\ln\Big(1-\frac{1-\varepsilon}{1+\Delta}\Big)+q_0+\ln(1-q_0)\\
&\quad\leq-\alpha_1\ln\Big(1-\frac{g_\sigma(G)}{1+\Delta}\Big)+q_0+\ln(1-q_0),\qquad\text{by }(\ast)\\
&\quad=\alpha_1\,\psi(q_0)+q_0+\ln(1-q_0),
\end{align*}
where $\psi(q):=-\ln\big(1-g_\sigma(g_\sigma(q))/(1+\Delta)\big)$, and $(\ast)$ uses
\begin{align*}
1-\varepsilon=\sum_{i=1}^{k_2}v_i^2g_\sigma(Gq_i)\leq g_\sigma(G)\sum_{i=1}^{k_2}v_i^2=g_\sigma(G).
\end{align*}
Since $g_\sigma(q)=O(q^3)$, $\psi(q)=O(q^9)$ as $q\to0$, while $-q-\ln(1-q)\simeq q^2/2$; and $\psi$ is bounded on $[0,1]$, while $-q-\ln(1-q)\to\infty$ as $q\to1$. Hence
\begin{align*}
\alpha_{1c}:=\inf_{q\in(0,1)}\frac{-q-\ln(1-q)}{\psi(q)}>0.
\end{align*}
If $\alpha_1<\alpha_{1c}$, then $\alpha_1\psi(q_0)<-q_0-\ln(1-q_0)$ for every $q_0\in(0,1)$, so $F(q_0,\bq)<F(0,\boldsymbol0)$: the global maximizer is the trivial point, and $\varepsilon^{\rm BO}=g_\sigma(1)=1$.

\emph{Lower bound.} Take $q_0=q_i=q:=1-1/\alpha_1$ for all $i$. Then $1-\varepsilon=g_\sigma(g_\sigma(q)q)\to1$ as $\alpha_1\to\infty$, so $\varepsilon\leq1/2$ for $\alpha_1$ large enough, and
\begin{align*}
F(q_0,\bq)-F(0,\boldsymbol0)
&=\alpha_1\ln\frac{1+\Delta}{\Delta+\varepsilon}+\Big(1+\frac{k_2}{d}\Big)\big(q+\ln(1-q)\big)\\
&=\alpha_1\ln\frac{1+\Delta}{\Delta+\varepsilon}-\Big(1+\frac{k_2}{d}\Big)\Big(\ln\alpha_1-1+\frac{1}{\alpha_1}\Big)\\
&\geq\alpha_1\ln\frac{1+\Delta}{\Delta+1/2}-\Big(1+\frac{k_2}{d}\Big)\ln\alpha_1,\qquad\text{since }\varepsilon\leq1/2,
\end{align*}
which is positive for $\alpha_1$ large enough, since $k_2/d=O(1)$.

Both bounds hold for any readout. Learning therefore starts at $n=\Theta(k_1d)$, through a joint discontinuous jump of $q_0$ and of the top second-layer overlaps. This is the two-layer analogue of the first feature transition at $n=\Theta(d/v_1^2)$ for one layer, but its location does not depend on the readout. At the transition $\snr_i\asymp d\,v_i^2$, so the second-layer rows learned at once are those with $v_i^2\gtrsim1/d$; for power-law readouts, there are $k_c\asymp d^{1/(2\beta)}\gg1$ of them.

\subsection{Error for \texorpdfstring{$n\gg k_1d$}{n >> k1 d}}\label{sec:2layer-err}
\paragraph{An exact decomposition.}
Split the error of each row into the error of the first layer and that of the row itself:
\begin{align}
\varepsilon&=\sum_{i=1}^{k_2}v_i^2\big(1-g_\sigma(Gq_i)\big)\nonumber\\
&=\sum_{i=1}^{k_2}v_i^2\big(1-g_\sigma(G)\big)+\sum_{i=1}^{k_2}v_i^2\big(g_\sigma(G)-g_\sigma(Gq_i)\big)\nonumber\\
&=1-g_\sigma(G)+\sum_{i=1}^{k_2}v_i^2\big(g_\sigma(G)-g_\sigma(Gq_i)\big),\qquad\text{since }\textstyle\sum_iv_i^2=1\nonumber\\
&=\underbrace{1-g_\sigma(G)}_{\text{first layer}}
+\underbrace{(\Delta+\varepsilon)\frac{k_1}{n}\sum_{i=1}^{k_2}e_G(\snr_i)}_{\text{second layer}},\label{eq:decomp}
\end{align}
where the last line uses the definition~\eqref{eq:fpe2} of the signal-to-noise ratio of row $i$, $\snr_i=(n/k_1)v_i^2/(\Delta+\varepsilon)$, in the form $v_i^2=(\Delta+\varepsilon)\frac{k_1}{n}\snr_i$, together with $q_i=m_G(\snr_i)$, and
\begin{align*}
e_G(x):=x\big(g_\sigma(G)-g_\sigma(G\,m_G(x))\big).
\end{align*}
Row $i$ contributes $(\Delta+\varepsilon)\frac{k_1}{n}e_G(\snr_i)=\Theta(\frac{k_1}{n}e_G(\snr_i))$:
\begin{itemize}
\item \emph{non-learnable row}, $\snr_i<\lambda_G$: $m_G(\snr_i)=0$, so the row contributes $g_\sigma(G)v_i^2=\Theta(v_i^2)$;
\item \emph{learned row}, $\snr_i>\lambda_G$: $e_G(\snr_i)=\Theta(1)$, by Appendix~\ref{app:msig} applied to $\sigma_G$, uniformly as $G\to1$. The row contributes $\Theta(k_1/n)$: its number of weights divided by $n$.
\end{itemize}

\paragraph{The threshold $\lambda_G$.} Above the first transition, $q_0$ is bounded away from $0$, hence so is $G=g_\sigma(q_0)$. Then $\lambda_G=\Theta(1)$, and $\lambda_G\to\lambda_\sigma$ as $G\to1$. Indeed, since $g_\sigma$ has non-negative coefficients starting at order $3$, $g_\sigma(Gq)\leq G^3g_\sigma(q)$, and $g_\sigma(Gq)$ is non-decreasing in $G$, so $\lambda_G$ is non-increasing in $G$. For $0<G_0\leq G\leq1$,
\begin{align*}
\lambda_G&=\inf_{q\in(0,1)}\frac{-q-\ln(1-q)}{g_\sigma(Gq)}\\
&\geq\inf_{q\in(0,1)}\frac{-q-\ln(1-q)}{G^3g_\sigma(q)},\qquad\text{since }g_\sigma(Gq)\leq G^3g_\sigma(q)\\
&=G^{-3}\lambda_\sigma\geq\lambda_\sigma,
\end{align*}
and $\lambda_G\leq\lambda_{G_0}<\infty$. With $q_\sigma$ the minimizer of Appendix~\ref{app:msig}, at which the infimum $\lambda_\sigma$ is reached,
\begin{align*}
\lambda_\sigma\leq\lambda_G
&\leq\frac{-q_\sigma-\ln(1-q_\sigma)}{g_\sigma(Gq_\sigma)}\\
&\xrightarrow[G\to1]{}\frac{-q_\sigma-\ln(1-q_\sigma)}{g_\sigma(q_\sigma)}=\lambda_\sigma.
\end{align*}

\paragraph{Scaling of the error.} We bound the three parts of~\eqref{eq:decomp} in turn.
\begin{itemize}
\item \emph{Learned rows.} There are $k_c$ of them, each contributing $\Theta(k_1/n)$: in total $\Theta(k_ck_1/n)$.
\item \emph{Non-learnable rows.} Each has $\snr_i<\lambda_G$, that is $v_i^2<\lambda_G(\Delta+\varepsilon)k_1/n=O(k_1/n)$. There are $k_2-k_c$ of them, so their total contribution is $O((k_2-k_c)k_1/n)$. Since $k_2=\Theta(d)$, their total readout mass is $O(k_2k_1/n)=O(1/\alpha_1)$.
\item \emph{First layer.} The learned rows thus carry readout mass $1-O(1/\alpha_1)$, and their overlaps are at least the jump value $q_G$. Since $g_\sigma'$ is non-decreasing, the sum $S$ of~\eqref{eq:q0} satisfies
\begin{align*}
g'_\sigma(1)\geq S
&\geq\sum_{i:\,\snr_i>\lambda_G}v_i^2\,q_i\,g_\sigma'(Gq_i)\\
&\geq\big(1-O(1/\alpha_1)\big)\,q_G\,g_\sigma'(Gq_G),
\end{align*}
so $S=\Theta(1)$, and~\eqref{eq:q0} gives
\begin{align*}
\frac{q_0}{1-q_0}=\alpha_1\,\frac{g_\sigma'(q_0)\,S}{\Delta+\varepsilon}=\Theta(\alpha_1),
\end{align*}
that is $1-q_0=\Theta(1/\alpha_1)$. As $g_\sigma(1)=1$ and $g'_\sigma=\Theta(1)$ on $[G,1]$, the mean value theorem gives
\begin{align*}
1-g_\sigma(G)&=\Theta(1-G) =\Theta\big(1-g_\sigma(q_0)\big)\\
&=\Theta(1-q_0) =\Theta(1/\alpha_1)=\Theta(k_1d/n).
\end{align*}
\end{itemize}
Adding the three parts,
\begin{align}
\varepsilon^{\rm BO}&=\underbrace{\Theta\Big(\frac{k_1d}{n}\Big)}_{\text{first layer}}+\underbrace{\Theta\Big(\frac{k_ck_1}{n}\Big)}_{\text{learned rows}}+\underbrace{O\Big(\frac{(k_2-k_c)k_1}{n}\Big)}_{\text{non-learnable rows}}\nonumber\\
&=\Theta\Big(\frac{k_1d}{n}\Big),\qquad\text{since }k_c\leq k_2=\Theta(d),\label{eq:main2}
\end{align}
for every readout, as stated in Result~\ref{res:2layer}. In particular $\varepsilon\to0$ and $G\to1$ for $n\gg k_1d$.

For one layer, the same decomposition gives $\varepsilon^{\rm BO}=\Theta(k_cd/n)+\sum_{i>k_c}v_i^2$, with no fixed term: the readout sets the exponent through the growth of $k_c$ with $n$, giving $1/(2\beta)-1$ for power-law readouts. With two layers, the first layer is learned as a whole and costs its $k_1d$ weights from the first transition; the second-layer rows, whose number $k_2=\Theta(d)$ is fixed, cost at most a constant times as much. The readout only changes the constant in~\eqref{eq:main2}, not the exponent.

\subsection{Power-law readouts}\label{sec:2layer-pl}
Let $v_i^2\propto i^{-2\beta}$ with $\beta>1/2$ and $\|\bv\|=1$. Since $\sum_{j\geq1}j^{-2\beta}<\infty$, the normalization is $\Theta(1)$, and $v_i^2=\Theta(i^{-2\beta})$ uniformly in $k_2$. By~\eqref{eq:fpe2}, row $i$ is learned if and only if $\snr_i\geq\lambda_G$, that is
\begin{align*}
v_i^2\geq\lambda_G\,(\Delta+\varepsilon)\,\frac{k_1}{n}.
\end{align*}
Since $v_i$ is non-increasing, the learned rows are the first $k_c$. Above the first transition, $\lambda_G=\Theta(1)$ (Section~\ref{sec:2layer-err}), and $\Delta\leq\Delta+\varepsilon\leq\Delta+1$. The threshold on $v_i^2$ is therefore $\Theta(k_1/n)$, and
\begin{align}
k_c&=\min\big\{k_2,\ |\{i:\,v_i^2\geq\Theta(k_1/n)\}|\big\}\nonumber\\
&=\min\big\{k_2,\ |\{i:\,i^{-2\beta}\geq\Theta(k_1/n)\}|\big\}\nonumber\\
&=\min\big\{k_2,\Theta\big((n/k_1)^{1/(2\beta)}\big)\big\}.\label{eq:kc2}
\end{align}
This is the one-layer law $k_c=\Theta((n/d)^{1/(2\beta)})$, with $d$ replaced by $k_1$. At the first transition, $n\asymp k_1d$ gives $k_c\asymp d^{1/(2\beta)}\gg1$, as found in Section~\ref{sec:nolearn}. The refinement regime starts when the weakest row becomes learnable, $\snr_{k_2}=\lambda_\sigma$ (by then $G\to1$, so $\lambda_G\to\lambda_\sigma$), at $n\asymp k_1k_2^{2\beta}$, or $\alpha_1\asymp d^{2\beta-1}$.

\subsection{HMC check}\label{sec:2layer-hmc}
\begin{figure}[t]
\centering
\includegraphics[width=0.6\linewidth]{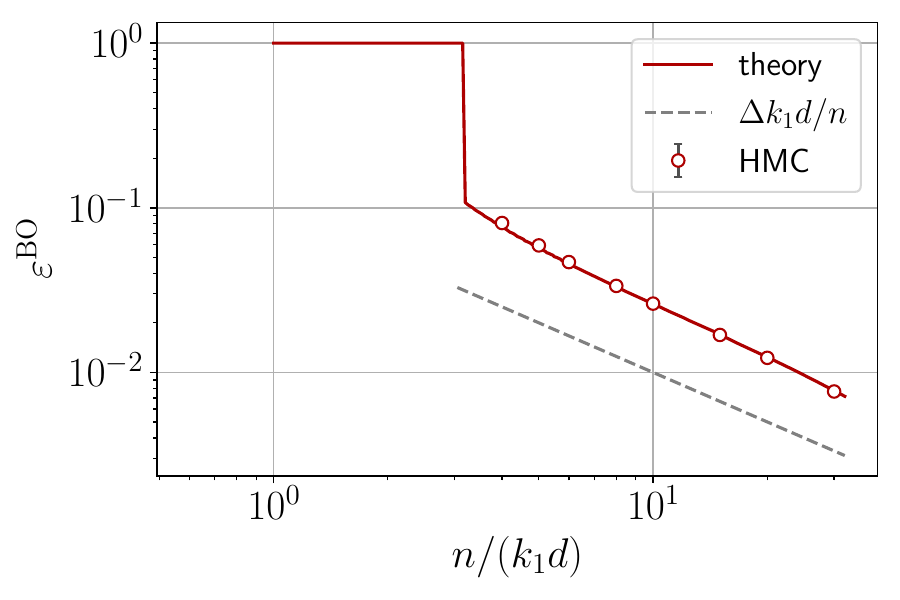}
\vspace{-5pt}
\caption{Bayes-optimal error of the two-hidden-layer model~\eqref{eq:2layer} vs.\ $n/(k_1d)$: theory from the potential~\eqref{eq:F} (solid) and HMC (circles; half the Gibbs error, mean $\pm$ standard deviation over 6 runs). Dashed: $\Delta k_1d/n$. $d=k_1=k_2=100$, $\beta=1$, $\Delta=0.1$, $\sigma(x)\propto\tanh(2x)-\mu_1x$ with $g_\sigma(1)=1$.  The theoretical error also jumps each time a new second-layer row is learned, but these jumps are too small to be seen in the figure. The first transition learns the first layer and the top $k_c$ rows at once, and removes most of the error. The rows learned afterwards have small readouts, $v_i^2=\Theta(k_1/n)$, so each jump is $\Theta(k_1/n)$, a fraction $\Theta(1/d)$ of $\varepsilon^{\rm BO}=\Theta(k_1d/n)$.}
\label{fig:2layer-hmc}
\end{figure}

The HMC experiment is similar to that for the one-hidden-layer case (Appendix \ref{app:hmc}). The initialization is partially informative: $\bW^{(1)}$ and the top $k_c$ rows of $\bW^{(2)}$ start at the teacher, with $k_c$ from theory, and the other rows of $\bW^{(2)}$ start at random. The check therefore does not validate $k_c$ itself or the location of the transition.

\newpage
\end{document}